\documentclass[preprint,11pt]{elsarticle}

\usepackage[toc,page]{appendix}
\usepackage{courier}
\usepackage{pdflscape}
\usepackage{algorithm}
\usepackage{tikz}
\usepackage{packages/gnuplot-lua-tikz}
\usepackage{graphicx}
\usepackage{longtable}
\usepackage{amsmath}
\usepackage{amssymb}
\usepackage{amsfonts}
\usepackage{mathtools}
\usepackage{amsthm}
\usepackage{verbatim}
\usepackage{array}
\usepackage{listings}
\usepackage[outline]{contour}
\usepackage{a4wide}
\usepackage{realboxes}
\usepackage{thmtools,thm-restate}

\newtheorem{definition}{Definition}
\newtheorem{example}{Example}

\usepackage[font=small,labelfont=bf,justification=centering,labelsep=space,singlelinecheck=off]{caption}
\usepackage[font=footnotesize]{subcaption}

\newcommand{\ignore}[1]{}

\newcommand{\fixme}[1]{\textcolor{red}{#1}}

\newcolumntype{L}{>{$\displaystyle}l<{$}}
\newcolumntype{C}{>{$\displaystyle}c<{$}}
\newcolumntype{R}{>{$\displaystyle}r<{$}}

\begin{document}

\begin{frontmatter}

    \title{Weighted Positive Binary Decision Diagrams\\for Exact Probabilistic Inference}

    \author{Giso H. Dal\ \ and\ \ Peter J.F. Lucas,\\
        Institute for Computing and Information Sciences, \\
        Radboud University Nijmegen, The Netherlands \\
        \texttt{\{gdal,peterl\}@cs.ru.nl}
    }

    \begin{abstract}
    Recent work on weighted model counting has been very successfully applied to the problem of probabilistic inference in Bayesian networks. The probability distribution is encoded into a Boolean normal form and compiled to a target language, in order to represent local structure expressed among conditional probabilities more efficiently. We show that further improvements are possible, by exploiting the knowledge that is lost during the encoding phase and incorporating it into a compiler inspired by Satisfiability Modulo Theories. Constraints among variables are used as a background theory, which allows us to optimize the Shannon decomposition. We propose a new language, called \emph{Weighted Positive Binary Decision Diagrams}, that reduces the cost of probabilistic inference by using this decomposition variant to induce an arithmetic circuit of reduced size.
\end{abstract}

\end{frontmatter}

\section{Introduction}\label{sec:introduction}

Bayesian networks, BNs for short, have been a subject of great interest partly due to their contribution in solving real-life problems that involve uncertainty. Bayesian networks are probabilistic
graphical models that represent joint probability distributions concisely by factoring them into conditional probabilities based on independence assumptions, in order to perform inference more efficiently~\cite{pearl1988probabilistic}. Further representational and
computational advances have been made by exploiting causal independence~\cite{heckerman96}, as well as contextual independence~\cite{boutilier1996context} and determinism~\cite{jensen2013approximations} expressed in conditional probability tables (CPTs). In order to capture these independencies local to CPTs, Bayesian networks have been represented as weighted Boolean formulas~\cite{bacchus2003dpll,chavira2008probabilistic}, reducing inference to \emph{Weighted Model Counting}~(WMC), or \emph{weighted~\#SAT}~\cite{bacchus2003dpll}. By representing a Bayesian network as Boolean formula $f$
in \emph{conjunctive normal form} (CNF), it can be compiled into a more concise normal form, or language, that renders inference a polytime operation in the size of the representation~\cite{wachter2007logical}.

A joint probability space with $n$ Boolean variables has $2^n$
interpretations. It is therefore necessary to be able to reason with
sets of interpretations, requiring a symbolic representation
\cite{burch1990symbolic}. Symbolic inference unifies the work of
probabilistic inference and the extensive research done in the field
of model checking, verification and satisfiability
\cite{shachter1990symbolic}. \emph{Ordered Binary Decision Diagrams}
(OBDDs) are based on Shannon decompositions and have been a very
influential symbolic representation that reduces compilation to the
problem of finding the variable ordering resulting in the optimal
factoring.

We focus on the disadvantage of the approaches in recent work that encode a BN as an independent CNF $f$, motivated by the ability to use off-the-shelf SAT-solvers. While maintaining this ability, we exploit the knowledge that is lost during the encoding without requiring this independence.

Our contributions are the following. We propose a weighted variant of
OBDDs, called \emph{Weighted Positive Binary Decision Diagrams}
(WPBDDs), which are based on \emph{positive Shannon decompositions},
allowing constraints in BNs to be represented more concisely. We use
probabilities as \emph{symbolic} edge weights, reducing the
search space exponentially. An optimized compilation
algorithm is introduced, inspired by the field of Satisfiability Modulo Theories
(SMT), namely a \emph{lazy SMT-solver}
\cite{barrett2009satisfiability}. It provides the means to view
constraints among variables in the encoding as background theory
$\mathcal{T}$ which supports the SAT-solver, allowing \emph{constant time conditioning}. We compile the conditional
probability tables of a BN explicitly,
but leave implicit the domain closure implied by the encoding. This
approach allows us to remove by up to a third of the clauses in the
encoding.

A comparison is provided with the state-of-the-art CUDD
(CU Decision Diagram \cite{cudd2015}) and SDD
(Sentential Decision Diagram \cite{choi2013compiling})
compilers and we show that WPBDDs induce arithmetic circuits that are
60\% reduced in size \emph{on average} compared to a corresponding
OBDD circuits at representing over 30 publicly available BNs. We show
an inference speedup of over 2.6 times on average compared to Weighted
Model Counting with OBDDs, and an average speedup between 5 and 1000 times compared to different implementations of the Junction Tree algorithm, making WPBDDs a valuable addition to the
field of exact probabilistic inference.

After preliminaries and background (Section~\ref{sec:background}), we introduce WPBDDs (Section~\ref{sec:wpbdd}). The process of using a BN to perform exact inference by WMC is explained (Section~\ref{sec:symbolicinference}) in addition to its optimization (Section~\ref{sec:optimizations}). We conclude with experimental results (Section~\ref{sec:results}), and review achievements (Section \ref{sec:conclusion}).
First we summarize related work.

\section{Related Work}\label{sec:relatedwork}

Probabilistic inference is a hard computational problem that can be achieved by marginalizing out non-evidence variables from a joint distribution, requiring an exponential number of operations in the worst case. Efforts toward efficient exact probabilistic inference attempt to find a concise factorization, e.g., BNs represent a joint probability distribution as a multiplicative factorization, exploiting independencies among variables. Further improvements to this factorization have been made by using propositional and first order logic~\cite{nilsson1986probabilistic,halpern1990analysis,poole2003first}. Representing probability distributions as Boolean functions empowers probabilistic inference with the tools developed for VLSI-CAD design, by using symbolic representations and Boolean algebra for minimization. \emph{Symbolic Probabilistic Inference} (SPI) \cite{shachter1990symbolic} is a good example of this, which is currently more commonly referred to as inference by WMC or \#SAT~\cite{bacchus2003algorithms}.

A BN can be viewed as a constraint satisfaction problem (CSP) and
translated to a satisfiability (SAT) instance in \emph{conjunctive
  normal form} (CNF), a form commonly used in satisfiability
solving. Various encodings of CSPs have been proposed, namely log,
direct \cite{walsh2000sat}, order \cite{bailleux2003efficient},
compact order \cite{tanjo2011compact}, log-support encoding
\cite{gavanelli2007log}, etc. In the context of BNs, a probability
distribution can be considered a pseudo Boolean function $f: \{0,1\}^n
\rightarrow \mathbb{R}$, with arity $n$, which can be uniquely written
as an exponentially sized multi-linear polynomial
\cite{minato2007compiling,poon2011sum}. Others have used the direct
encoding \cite{chavira2008probabilistic}, or a combination between the
direct and order encoding \cite{sang2005performing}, where a BN is
viewed as a set of discrete real valued functions, where each function
represents a distinct CPT.

Inference by WMC is motivated by linear time complexity in the size of the representation~\cite{bozga1999representation}, where the common goal is to exploit local structure \cite{zhang1999role}. Choosing a representation or \emph{language} to compile to is therefore a critical task, where one must deal with the balance between the functions a language can represent concisely versus its algorithmic properties. Initial attempts include probability trees~\cite{boutilier1996context,cano2000penniless,kozlov1997nonuniform} and recursive (factored) probability trees~\cite{cano2009recursive}, which focus on concisely representing each CPT independently, allowing their usage in inference algorithms directly. Probabilistic Decision Graphs (PDG) have even shown that the smallest PDG is at least as small as the smallest Junction tree for the same distribution~\cite{jaeger2004probabilistic}.

Current WMC approaches to inference divide into \emph{search} and \emph{compilation} methods \cite{gomes2008model}. Typical search algorithms are based on DPLL-style SAT solvers that do an exhaustive run to count all satisfying models \cite{sang2005performing}. Recording SAT evaluation paths (i.e. resolution steps) as a compiled structure (e.g. an OBDD), yields one possible factoring. We refer to finding the optimal factoring given all variable orderings, as \emph{exact} compilation.

Compilation performance has been improved by clause learning \cite{beame2003memoization}, formula caching \cite{majercik1998using}, bounding \cite{fischer2008counting}, and using canonical languages. Representational advances include symmetry detection \cite{sen2009prdb,kersting2009counting}, support for causal
independence \cite{li2008exploiting} and using \emph{read-once functions}\cite{sen2010read}.

Representations relevant in the context of BN compilation are \emph{AND/OR Multi-Valued Decision Diagrams} (AOMDD) \cite{mateescu2008and}, \emph{Sentential Decision Diagrams} (SDD) \cite{choi2013compiling}, Zero-suppressed Binary decision diagrams (ZBDD) \cite{minato2007compiling} and Ordered Binary Decision Diagrams (OBDD) \cite{nielsen2000using}, which view probabilities as auxiliary literals, resulting in an intractably large search space. Multi-Terminal BDDs  \cite{clarke1993spectral} represent multi-valued functions, but would require too many terminal nodes considering the size of a probability distribution. Variants of Edge-Value BDDs \cite{lai1992edge,sanner2005affine} focus on real valued functions. When multiple CPTs have probabilities in common, we lose the ability to distinguish from which CPT the probabilities originated. We therefore cannot determine on which variables they depend, resulting in an inconsistent model count with regard to the distribution. Our approach to maintain consistency is to represent probabilistic edges weights \emph{symbolically}. This differentiates our approach from Multi-Terminal BDDs \cite{clarke1993spectral} and Edge-Value BDDs \cite{lai1992edge}. And unlike SDDs \cite{darwiche2011sdd}, we are not obligated to view probabilities as auxiliary literals, reducing the search space to a fraction of its former size. A common characteristic with ZBDDs, is the ability to represent mutual-exclusive constraints more concisely \cite{minato1993zero}. The intuitive difference is that ZBDD optimize only the positive cofactor, while we optimize both the positive and negative cofactor of decomposition nodes, a matter we will elaborate on in upcoming discussions.

\section{Preliminaries and Background}\label{sec:background}
We provide here a description of what Bayesian networks are and introduce a running example (Section~\ref{sec:bayes}), show how to encode a BN onto the Boolean domain (Section~\ref{sec:encoding}), and describe an influential representation that will serve for comparison with ours (Section~\ref{sec:obdd}).

\subsection{Bayesian Networks}\label{sec:bayes}

Probabilistic inference is an important computational problem in Artificial Intelligence. A full joint probability distribution defined over $n$ Boolean variables is of size $\mathcal{O}(2^n)$. Finding the minimal representation of a function describing a probability distribution reduces memory and inference complexity, which is the motivation for this paper.

A Bayesian Network (BN) is a graphical representation that is used to compactly represent a joint distribution as a product of factors, by taking advantage of \emph{conditional independence} (CI). A BN is an directed acyclic graph (DAG) that models variables $X$ as nodes, the dependencies among them as edges, and their joint probability distribution as \[P(X) = P(x^1,...,x^n) = \prod^{n}_{i=1} P(x^i\ |\ pa(x^i)),\] where $P(x^i\ |\ pa(x^i))$ represents the conditional probability of variable $x^i$ given its parents $pa(x^i)$. Conditional probability tables (CPTs) are associated with edges and capture the degree to which variables are related. BNs reduce the size of representing a probability distribution to $\mathcal{O}(n2^k)$, where $k$ is the maximum number of parents of any node. 

\begin{example}\label{ex:full}
    Figure~\ref{fig:example} shows BN $\mathcal{B}$ defined over variables $X = \{a,b\}$ (Figure~\ref{subfig:bn}), its CPTs (Figure~\ref{subfig:cpts}) and corresponding full joint probability distribution (Figure~\ref{subfig:joint}).
    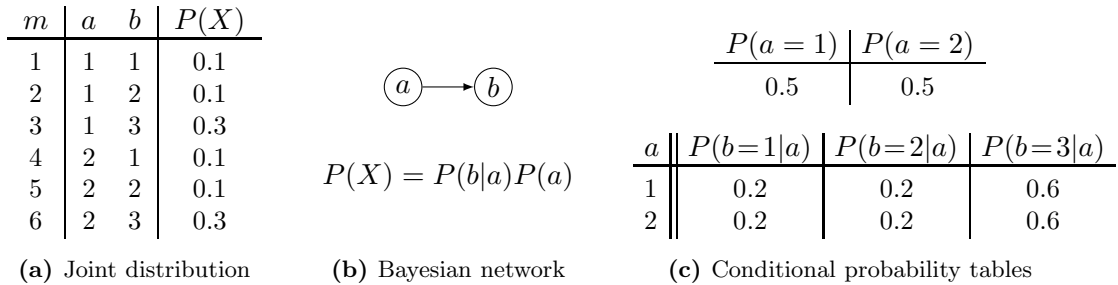
\begin{figure}[H]
    \centering
    \begin{subfigure}[b]{0.22\textwidth}
        \centering
        \begin{small}
            \begin{tabular}{c | c c | c }
                \normalsize{$m$} & \normalsize{$a$} & \normalsize{$b$} & \normalsize{$P(X)$}\\\hline
                &&&\\[-2ex]
                1  & 1 & 1 & 0.1\\
                2  & 1 & 2 & 0.1\\
                3  & 1 & 3 & 0.3\\
                4  & 2 & 1 & 0.1\\
                5  & 2 & 2 & 0.1\\
                6  & 2 & 3 & 0.3\\
            \end{tabular}
        \end{small}
        \caption{Joint distribution}
        \label{subfig:joint}
    \end{subfigure}
    \begin{subfigure}[b]{0.3\textwidth}
        \centering
        \begin{tikzpicture}[
            scale=0.2,
            every path/.style={>=latex},
            every node/.style={draw},
            inner sep=0pt,
            minimum size=0.5cm,
            line width=1pt,
            thin,
            font=\normalsize
            ]

            \node[circle] (a) at (0,0)  {$a$};
            \node[circle] (b) at (6,0)  {$b$};
            \node[draw=none] (d) at (3,-6)  {$P(X) = P(b|a)P(a)$};
            \node[draw=none] (c) at (0,-9)  {};

            \draw[->] (a) edge (b);
        \end{tikzpicture}
        \caption{Bayesian network}
        \label{subfig:bn}
    \end{subfigure}
    \begin{subfigure}[b]{0.37\textwidth}
        \begin{minipage}{\textwidth}
            \begin{minipage}{\textwidth}
                \centering
                \setlength{\tabcolsep}{4pt}
                \begin{small}
                    \begin{tabular}[t]{c | c}
                        \normalsize{$P(a=1)$} & \normalsize{$P(a=2)$} \\\hline
                                              &\\[-2ex]
                        0.5 & 0.5\\
                    \end{tabular}\vspace{1em}\\
                \end{small}
            \end{minipage}\\
            \begin{minipage}{\textwidth}
                \begin{small}
                    \setlength{\tabcolsep}{4pt}
                    \begin{tabular}[t]{l || c | c | c }
                        \normalsize{$a$} & \normalsize{$P(b\!=\!1 | a)$} & \normalsize{$P(b\!=\!2 | a)$} & \normalsize{$P(b\!=\!3 | a)$}\\\hline
                                         &&&\\[-2ex]
                        1 &  0.2 & 0.2 & 0.6\\
                        2 &  0.2 & 0.2 & 0.6\\
                    \end{tabular}
                \end{small}
            \end{minipage}
        \end{minipage}
        \caption{Conditional probability tables}
        \label{subfig:cpts}
    \end{subfigure}

    \caption{Bayesian network with local structure}
    \label{fig:example}
\end{figure}

\end{example}

Figure~\ref{subfig:bn} includes a factored form that can greatly be improved when CPTs exhibit \emph{local structure}, which comes in two forms. \emph{Context-specific independence} (CSI) is expressed when probabilities in a CPT show uniformity regardless of the value of one or more variables they have in common, with or without a certain context. \emph{Determinism} is expressed when probabilities in a CPT are equal to 0 or 1, which can be used to simplify the Boolean formula that represents it.

In order to exploit more of the problem structure than a BN, we \emph{compile} it to a target language that is more capable of doing so. The compilation process is not just about reformulating into a different language, it is also about finding the minimal representation given that language. The goal is to derive an arithmetic circuit corresponding to that representation with improved complexity compared to the standard factorization.

\subsection{Encoding Bayesian Networks}\label{sec:encoding}

In order to exploit local structure, we encode BNs into \emph{conjunctive normal form} (CNF), which is the most common representation used in satisfiability solving. It consists of a conjunction of clauses, where each clause is a disjunction of literals. A literal is a propositional Boolean variable or its negation. A Bayesian Network defined over variables $X$ can be seen as a multi-linear function $f : X \rightarrow \mathbb{R}$. Using a weighted adaptation of the \emph{sparse} or \emph{direct} encoding \cite{chavira2008probabilistic,walsh2000sat,hoos1999sat}, we encode $f$ into a Boolean function $\mathcal{E}(f) = f^{e}$ by representing it as a weighted CNF:

\begin{equation}
    \mathcal{E}(f) = f^c \land f^m,
\end{equation}
where constraint clauses $f^c$ support the mapping $\mathcal{M}(f) = f^m$ that encodes probabilities and introduces a Boolean variable for each unique variable-value pair. Details are discussed below.

\subsubsection{Encoding Constraints}

The mapping function $\mathcal{M}$ introduces for each $x \in X$ atoms $\mathcal{A}(x) = \{x_1,\dotsc,x_n\}$, where $x_i$ signifies $x$ being equal to its $i^{th}$ value. To maintain consistency among variables we add to $f^c$ an \emph{at-least-once} (ALO) constraint clause for each $x$, to ensure $x$ is assigned a value:
\begin{equation}\label{eq:variables}
    (x_1 \lor \cdots \lor x_n)
\end{equation}
As values of a variable are mutually exclusive, we add to $f^c$ the following \emph{at-most-once} (AMO) constraint clauses:

\begin{equation}\label{eq:constraints}
    \bigwedge\limits_{i = 1}^{n}\ \ \left(x_i \implies \bigwedge\limits_{x_j\in \mathcal{A}(x) \setminus \{x_i\}} \overline{x_j}\right)\ \ =\ \ \bigwedge\limits_{i = 1}^{n}\ \ \bigwedge\limits_{j = i + 1}^{n}\ \ (\overline{x_i} \lor \overline{x_j}),
\end{equation}
where $\overline{x_i}$ indicates the negation of $x_i$.
\subsubsection{Encoding CPTs}

The BN's factored form is preserved by using a weighted adaptation of the \emph{direct} encoding. The mapping function $\mathcal{M}$ adds a clause for every probability $P(x|U)$, where $x$ depends on variables $U = \{u^1,\dotsc,u^r\}$:
\begin{equation}\label{eq:probabilities}
        (x \land u^1 \land \cdots \land u^r \Rightarrow \omega_i)\ \ =\ \ (\overline{\vphantom{u^1}x} \lor \overline{\vphantom{u^1}u^1} \lor \cdots \lor \overline{\vphantom{u^1}u^r} \lor \omega_i),\\
\end{equation}
\noindent which we henceforth shall view as a \emph{weighted clause} $(\overline{\vphantom{u^1}x} \lor \overline{\vphantom{u^1}u^1} \lor \cdots \lor \overline{\vphantom{u^1}u^r}):\omega_i$, where its symbolic weight $\omega_i$ represents probability $P(x|U)$. We introduce a symbolic weight for every unique probability per CPT, and thus allow multiple models to be associated with it:

\begin{equation}\label{eq:csi}
    {\strut}I^1 \land \cdots \land I^s,\\
\end{equation}
where each $I^j$ has weight $\omega_i$ and forms the \emph{implicate} of model $m$ associated with $P(\boldsymbol{x}|\textbf{\textit{U}})$, i.e., $I^j = \omega_i:(\boldsymbol{x} \land \boldsymbol{u^{1}} \land \dotsc \land \boldsymbol{u^r})$, where the variables $\boldsymbol{x} \cup \textbf{\textit{U}}$ have been appropriately mapped by $\mathcal{M}$, i.e., $\boldsymbol{x} \in \mathcal{A}(x)$ and each $\boldsymbol{u^k} \in \mathcal{A}(u^k)$.

\begin{example}\label{ex:encoding}
Assume a BN as given in Example~\ref{ex:full}. The following clauses form $f^c$:

\begin{center}
    \begin{normalsize}
        \begin{tabular}{c |c| c}
            Variable & ALO (Eq.~\ref{eq:variables}) & AMO (Eq.~\ref{eq:constraints})\\\hline
            &\\[-2ex]
            $a$ & $(a_1 \lor a_2)$ & $(\overline{a_1} \lor \overline{a_2})$ \\
            $b$ & $(b_1 \lor b_2 \lor b_3)$ & $(\overline{b_1} \lor \overline{b_2}) \land (\overline{b_1} \lor \overline{b_3}) \land (\overline{b_2} \lor \overline{b_3})$\\
    \end{tabular}
    \end{normalsize}
\end{center}

The variables that make up the search space during compilation therefore are $\mathcal{A}(\{a,b\}) = \{a_1,a_2,b_1,b_2,b_3\}$. We encode equal probabilities $P(x|U)$ as unique symbolic weights per CPT:

\begin{figure}[H]
    \centering
    \begin{minipage}{0.15\textwidth}
        \centering
        \begin{normalsize}
            \setlength{\tabcolsep}{4pt}
            \begin{tabular}[t]{c | c}
                $P(a_1)$ & $P(a_2)$ \\\hline
                &\\[-2ex]
                $\omega_1$ & $\omega_1$\\
            \end{tabular}
        \end{normalsize}

    \end{minipage}
    \hspace{1em}
    \begin{minipage}{0.30\textwidth}
        \centering
        \begin{normalsize}
            \begin{tabular}[t]{c || c | c | c }
                \setlength{\tabcolsep}{2pt}
                $a$ & $P(b_1 | a)$ & $P(b_2 | a)$ & $P(b_3 | a)$\\\hline
                &&&\\[-2ex]
                1 & $\omega_2$ & $\omega_2$ & $\omega_3$\\
                2 & $\omega_2$ & $\omega_2$ & $\omega_3$\\
            \end{tabular}
        \end{normalsize}
    \end{minipage}
\end{figure}

\noindent In accordance with Equations~\ref{eq:probabilities} and \ref{eq:csi}, $f^m$ consists of the following clauses, accompanied by their respective symbolic weights:
\begin{center}
\begin{normalsize}
        $(\overline{\vphantom{b_3}a_1}):\omega_1\ \land
        \ (\overline{\vphantom{b_3}a_2}):\omega_2\ \land$ \\
        $(\overline{\vphantom{b_3}a_1} \lor \overline{b_1}):\omega_2\ \land
        \ (\overline{\vphantom{b_3}a_1} \lor \overline{b_2}):\omega_2\ \land
        \ (\overline{\vphantom{b_3}a_1} \lor \overline{b_1}):\omega_2\ \land$ \\
        $(\overline{\vphantom{b_3}a_1} \lor \overline{b_2}):\omega_2\ \land
        \ (\overline{\vphantom{b_3}a_1} \lor \overline{b_3}):\omega_3\ \land
        \ (\overline{\vphantom{b_3}a_2} \lor \overline{b_3}):\omega_3$.$\hphantom{\land}$
\end{normalsize}
\end{center}
\end{example}

\subsection{Ordered Binary Decision Diagrams}\label{sec:obdd}
A Boolean function $f$ defined over a set of variables $X$ is a function that maps each complete assignment of its variables to either \emph{true} (1) or \emph{false} (0). The \emph{conditioning} of $f$ on instantiated variable $\boldsymbol{x_i}$ is defined as the projection:
\begin{equation}\label{eq:projection}
    \begin{split}
        f_{|x_i \leftarrow b}(x_1,\ldots,x_n)
        = f(x_1,\ldots,x_{i-1},b,x_{i+1},\ldots,x_n),
    \end{split}
\end{equation}
with $b \in \{0,1\}$. We will use shorthand notations $f_{|x_i}$ and $f_{|\overline{x_i}}$ for $f_{|x_i \leftarrow 1}$ and $f_{|x_i \leftarrow 0}$, respectively. Shannon's theorem is used to find a more compact way to represent $f$ by \emph{factoring}~it.

\begin{restatable}{theorem}{theoremshannon}
\label{def:shannon}\cite{brown1990}
    \emph{Shannon's expansion} allows a Boolean function $f: \{0,1\}^n \rightarrow \{0,1\}$ defined over variables $X$, to be written in terms of its inputs:
    \begin{eqnarray}
        \nonumber f & = & x \land f_{|x}\ \ \lor\ \ \overline{x} \land f_{|\overline{x}},
    \end{eqnarray}
    with $x \in X$ and where $f_{|x}$ is called the positive \emph{cofactor} of $f$ with respect to $x$, and $f_{|\overline{x}}$ the negative cofactor. Applying Shannon's theorem is known as an \emph{(additive) decomposition} step. The \emph{decomposition} of $f$ is defined as the recursive application of Shannon's expansion theorem to cofactors, removing one variable at a time, until no variables are left.
\end{restatable}

Note that all proofs of theorems, lemmas, etc., can be found in the Appendix. The \emph{Shannon decomposition} is key to one of the most influential representations in Artificial Intelligence (AI), namely \emph{Ordered Binary Decision Diagrams} (OBDD)\cite{bryant1992symbolic}.

\begin{definition}\label{def:obdd}\cite{bryant1986graph}
    A \emph{Binary Decision Diagram} (BDD) represents Boolean function $f$ defined over variables $X$ as a rooted, directed acyclic graph, where each node $v$ represents a Shannon decomposition on variable $var(v) \in X$. A BDD is \emph{ordered} (OBDD) if variables appear in the same order on all paths from the root. It is a \emph{canonical} representation if it is \emph{reduced} by applying the following rules:
    \begin{enumerate}
        \item Merge rule: All isomorphic subgraphs are merged.
        \item Delete rule: All nodes are removed whose children are isomorphic.
    \end{enumerate}
\end{definition}

One can compile the described encoding of BNs to OBDDs to perform inference by WMC. Although other languages have been used in this context, we focus more on OBDDs as it is commonly
used in comparisons with related work.

\section{Weighted Positive Binary Decision Diagrams}\label{sec:wpbdd}
Consider variable $x$ and its corresponding mapped atoms $\mathcal{A}(x) = \{x_1,x_2\}$. Decision diagrams that are based on Shannon decompositions produce an unnecessarily large arithmetic circuit by redundantly representing constraints $f^c$ provided as part of encoding $\mathcal{E}$. They also do not take advantage of the symmetric relation $x_1 = \overline{x_2}$ and $\overline{x_1} = x_2$
in the presence of ALO constraints. To ameliorate this, we propose a new canonical language called \emph{Weighted Positive Binary Decision Diagrams} (WPBDD), which are based on \emph{positive Shannon decompositions} and \emph{implicit conditioning}. We will elaborate on these concepts through an intermediate unweighted variant of WPBDDs (PBDD).

\subsection{Explicit and Implicit Conditioning}
When encoded function $f^e$ contains a \emph{unit clause} $x_i$ (a clause consisting of a single literal), it can be simplified using \emph{unit propagation}:
    \begin{enumerate}
        \item Every clause containing $x_i$ is removed, excluding the unit clause.
        \item Literal $\overline{x_i}$ is removed from every clause containing it.
    \end{enumerate}
    When conditioning $f^e$ on literal $x_i$ we are guaranteed, due to constraints $f^c$, to obtain unit clauses containing negated literals $\overline{x_j}$, with $x_j \in \mathcal{A}(x){\backslash}x_i$ (i.e., if $x$ is equal to its $i^{th}$ value, it cannot be equal to its $j^th$ value).

\begin{example}\label{ex:unitclauses} We will show by example what unit clauses are obtained by conditioning on positive literals. Consider the constraint clauses provided by Example~\ref{ex:encoding}, regarding only variable $b$:\\[-0.5em]
    \[
     \begin{normalsize}
\setlength{\tabcolsep}{1pt}
 \begin{tabular}{ >{$}c<{$}>{$}c<{$}>{$}c<{$}  >{$}c<{$}  >{$}c<{$}  >{$}c<{$}  >{$}c<{$}  >{$}c<{$}  >{$}c<{$}  >{$}c<{$}  >{$}c<{$}  >{$}c<{$} }
     f^c & = & (b_1 \lor b_2 \lor b_3) &\land &(\overline{b_1} \lor \overline{b_2}) & \land &  (\overline{b_1} \lor \overline{b_3}) &  \land  &  (\overline{b_2} \lor \overline{b_3})\\
\end{tabular}
\end{normalsize}
\]\\[-0.5em]
According to Shannon's expansion the following holds:\\[-0.5em]

\[f^c = b_1 \land f^c_{|b_1}\ \ \lor\ \ \overline{b_1} \land f^c_{|\overline{b_1}}\]\\[-1.4em]

Now specifically look at the unit clauses that result from conditioning on positive literal $b_1$, i.e., instantiating $b_1$ and performing unit propagation.\\[-2em]
\begin{normalsize}
\begin{center}
\setlength{\tabcolsep}{3pt}
\[
\begin{tabular}{ >{$}c<{$}>{$}c<{$}>{$}c<{$}  >{$}c<{$}  >{$}c<{$}  >{$}c<{$}  >{$}c<{$}  >{$}c<{$}  >{$}c<{$}  >{$}c<{$}  >{$}c<{$}  >{$}c<{$} }
     f^c_{|b_1} & = & (1 \lor b_2 \lor b_3) &\land &(0 \lor \overline{b_2}) & \land &  (0 \lor \overline{b_3}) &  \land  &  (\overline{b_2} \lor \overline{b_3})\\
      & = & (1)  & \land  & (\overline{b_2}) &  \land & (\overline{b_3}) &  \land  &  (\overline{b_2} \lor \overline{b_3})\\
      & = & (1)  & \land & (\overline{b_2})  & \land &  (\overline{b_3})  & \\
      & = & (\overline{b_2})  & \land  & (\overline{b_3}).\\
\end{tabular}
\]
\end{center}
\end{normalsize}

\noindent Thus, conditioning on $b_i$ will result in unit clauses containing negated literals $\overline{b_j}$, with $b_j \in \mathcal{A}(b)\setminus \{b_i\}$.
\end{example}

We distinguish between two types of conditioning based on the previous observation, and describe them in the following definition.
\begin{definition}\label{def:implicitconditioning}
    Let $f^e$ be an encoded representation of function $f$, given encoding $\mathcal{E}$, where $f$ is defined over variables $X$. We define $f^e_{{\parallel}x_i}$ as the conditioning of $f^e$ on literals $\{x_i,\overline{x_1},\ldots,\overline{x_{i-1}},$ $\overline{x_{i+1}},\ldots,\overline{x_n}\}$ in any order, i.e., as the \emph{explicit} conditioning of $f^e$ on literal $x_i \in \mathcal{A}(x)$, and its \emph{implicit} conditioning on literals $\overline{x_j} \in \mathcal{A}(x){\backslash}x_i$, with $x \in X$:

    \[ f^e_{{\parallel}x_i} = f^e_{|x_i,\overline{x_1},\ldots,\overline{x_{i-1}},\overline{x_{i+1}},\ldots,\overline{x_n}}, \]\\[-1em]

\end{definition}

It follows from Definition~\ref{def:implicitconditioning} and the constraints provided by encoding $\mathcal{E}$, that the relation between $f^e_{{\parallel}x_i}$ and $f^e_{|x_i}$ is given by the following equality:
\begin{equation}\label{eq:equality}
    f^e_{|x_i} = \left( \bigwedge\limits_{\overline{x_j} \in \mathcal{A}(x){\backslash}x_i}\overline{x_j} \right) \land f^e_{{\parallel}x_i}.
\end{equation}

Implicit conditioning on unit clauses takes advantage of deterministic behavior expressed in $f^c$, while other representations would have to explicitly condition these unit clauses out. The advantage is two-fold. The size of the encoding can be reduced by removing constraint clauses $f^c$ generated by Equation~\ref{eq:variables} and \ref{eq:constraints}, and integrating them directly into the compilation process through theory $\mathcal{T}$. As will be shown later, by separating constraint clauses from $f^e$ as theory $\mathcal{T}$ allows them to be conditioned in \emph{constant time}, as opposed to quadratic time. Secondly, the size of the compiled structure is reduced by not having to represent redundant constraint information with our variant on the Shannon expansion, introduced in the following section.

\subsection{Positive Shannon Decomposition}

We propose \emph{positive Shannon decompositions} that use background knowledge to improve upon Shannon decompositions by combining it with implicit conditioning.

\begin{restatable}{lemma}{theorempositiveshannon}
\label{def:positiveshannon}
    A \emph{positive Shannon expansion} allows an \emph{encoded} Boolean function $\mathcal{E}(f) = f^e$, where $f$ is defined over $X$, to be written in terms of its inputs:
    \[
            f^e =\ \ f^c\ \land\ \ \left( x_i \land f^e_{{\parallel}x_i}\ \ \lor\ \ f^e_{|\overline{x_i}} \right),
    \]
    where $f^e = f^c \land f^m$, and $x_i \in \mathcal{A}(x)$, with $x \in X$. The positive cofactor $f^e_{{\parallel}x_i}$ incorporates implicit conditioning (Definition~\ref{def:implicitconditioning}). The negative cofactor $f^e_{|\overline{x_i}}$ and the \emph{decomposition} of $f^e$ are as defined by Theorem~\ref{def:shannon}.
\end{restatable}

As intuition might confirm, logical representations will grow by the reintroduction of constraints at every expansion. We introduce a \emph{reduced} expansion by removing constraint clauses $f^c$ that introduces additional models, in turn \emph{allowing us to find more concise representations}. These models can easily be removed by a post decomposition conjoin with $f^c$.

\begin{restatable}{theorem}{theoremreducedpositiveshannon}
\label{th:reducedpositiveshannon}
    A \emph{reduced positive Shannon expansion} allows an \emph{encoded} Boolean function $\mathcal{E}(f) = f^e$, where $f$ is defined over $X$, to be written in terms of its inputs under constraints $f^c$:
    \[
            f^e\ \ \models\ \ x_i \land f^e_{{\parallel}x_i}\ \ \lor\ \ f^e_{|\overline{x_i}},
    \]
    where $f^e = f^c \land f^m$, and $x_i \in \mathcal{A}(x)$, with $x \in X$. Cofactors and decomposition are as defined by Lemma~\ref{def:positiveshannon}.
\end{restatable}

By using the reduced form of the expansion, we are able represent constraints more concisely in corresponding logical circuits (Figure~\ref{fig:positiveshannon}), as well as in the soon to be introduced representation that utilizes it.

\begin{figure}[H]
     \centering
     \newcommand{\sepfig}{0.15}
     \begin{subfigure}[t]{\sepfig\textwidth}
         \centering
         \begin{tikzpicture}[
                 scale=0.3,
                 every path/.style={>=latex},
                 every node/.style={draw,circle},
                 inner sep=0pt,
                 minimum size=0.5cm,
                 line width=1pt,
                 thin,
                 font=\normalsize
                 ]

                 \node[draw=none] (or1) at (0,0) {$\vee$};
                 \node[draw=none] (and1) at (-3,-2)  {$\wedge$};
                 \node[draw=none] (and2) at (3,-2)  {$\wedge$};
                 \node[draw=none] (a1) at (-4.5,-5)  {$x_i$};
                 \node[draw=none] (a2) at (1.5,-5)  {$\overline{x_i}$};
                 \node[draw=none] (alpha) at (-1.5,-5)     {$f_{{\mid}x_i}$};
                 \node[draw=none] (beta) at (4.5,-5)     {$f_{{\mid}\overline{x_i}}$};

                 \draw[-] (or1) edge (and1);
                 \draw[-] (or1) edge (and2);
                 \draw[-] (and1) edge (a1);
                 \draw[-] (and2) edge (a2);
                 \draw[-] (and1) edge (alpha);
                 \draw[-] (and2) edge (beta);
         \end{tikzpicture}
        \captionsetup{width=1.5\textwidth}
         \caption{Shannon Logical Circuit}
         \label{fig:positivea}
    \end{subfigure}
    \hspace{4em}
    \begin{subfigure}[t]{\sepfig\textwidth}
        \centering
        \begin{tikzpicture}[
                scale=0.3,
                every path/.style={>=latex},
                every node/.style={draw,circle},
                inner sep=0pt,
                minimum size=0.5cm,
                line width=1pt,
                thin,
                font=\normalsize
            ]

            \node[draw=none] (or1) at (0,0) {$\vee$};
            \node[draw=none] (and1) at (-3,-2)  {$\wedge$};
            \node[draw=none] (f) at (3,-2)  {$f_{|\overline{x_i}}$};
            \node[draw=none] (a1) at (-4.5,-5)  {$x_i$};
            \node[draw=none] (alpha) at (-1.5,-5) {$f_{{\parallel}x_i}$};

            \draw[-] (or1) edge (and1);
            \draw[-] (or1) edge (f);
            \draw[-] (and1) edge (a1);
            \draw[-] (and1) edge (alpha);
        \end{tikzpicture}
        \captionsetup{width=1.5\textwidth}
        \caption{Positive Shannon Logical Circuit}
        \label{fig:positiveb}
    \end{subfigure}
    \caption{Logical Circuits}
    \label{fig:positiveshannon}
 \end{figure}
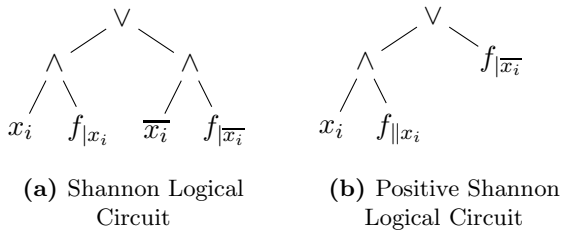

Using encoding $\mathcal{E}$ we can infer that the delete rule used to reduce OBDDs will never be applied to constraint clauses $f^c$. OBDDs are therefore not capable of capturing local structure along one dimension. To ameliorate this, we introduce \emph{positive} OBDDs (PBDD) as an unweighted intermediate representation, that are based on the positive Shannon decomposition and substitutes the delete rule with the \emph{collapse} rule, which as opposed to deleting literals, applies the \emph{distributive law} to involved literals in the induced logical circuit.

\begin{definition}\label{def:pbdd}
    A \emph{positive} OBDD (PBDD) represents Boolean function $\mathcal{E}(f) = f^e$, where $f$ is defined over variables $X$, as an \emph{ordered} BDD where each node $v$ represents a positive Shannon decomposition on variable $var(v) \in \mathcal{A}(X)$. It is a \emph{canonical} representation if \emph{reduced} by applying the following rules:
    \begin{enumerate}
        \item Merge rule: All isomorphic subgraphs are merged.
        \item Collapse rule: remove direct descendant $u$ of node $v$ iff $f_{{\parallel}x_i} = f_{{\parallel}x_j}$, where $var(v) = x_i$ and $var(u) = x_j$, with $x_i,x_j \in \mathcal{A}(x)$ and $x \in X$.
    \end{enumerate}
\end{definition}
A function \emph{essentially} depends on a variable if it appears in its prime implicate. The variable set $\mathcal{S}$, on which $f^e$ essentially depends, is called the \emph{support} of $f^e$. We will use this support set to identify to what variables the collapsed rule has been applied in order to produce its corresponding arithmetic circuit, a trick similarly utilized with Zero-Suppressed BDDs (ZBDD) \cite{minato1993zero}. Note that a Boolean function $\mathcal{E}(f) = f^e$, where $f$ is defined over variables $X$, essentially depends on $\mathcal{A}(X)$, because $f^c$ is a prime implicate that mentions all $\mathcal{A}(X)$. The canonical property of PBDDs follows from the fact that a binary tree can be reconstructed from a PBDD and its support set, by apply its reduction rules reversely.

Figure~\ref{fig:examplepositiveshannon} shows the difference in representational size between an OBDD, a ZBDD and a PBDD representing the same function with constraints on $\mathcal{A}(x) = \{x_1,x_2\}$, where $g$ is a Boolean function that does not essentially depend on literals $\{x_1,x_2\}$. The positive Shannon decomposition does not only reduce the size the corresponding logical circuit, it also reduces the size of the representation. More generally, OBDDs require an exponential number of nodes in the product of each constraint variable's dimension, where ZBDDs require a linear number, and where PBDDs require only 1 node.

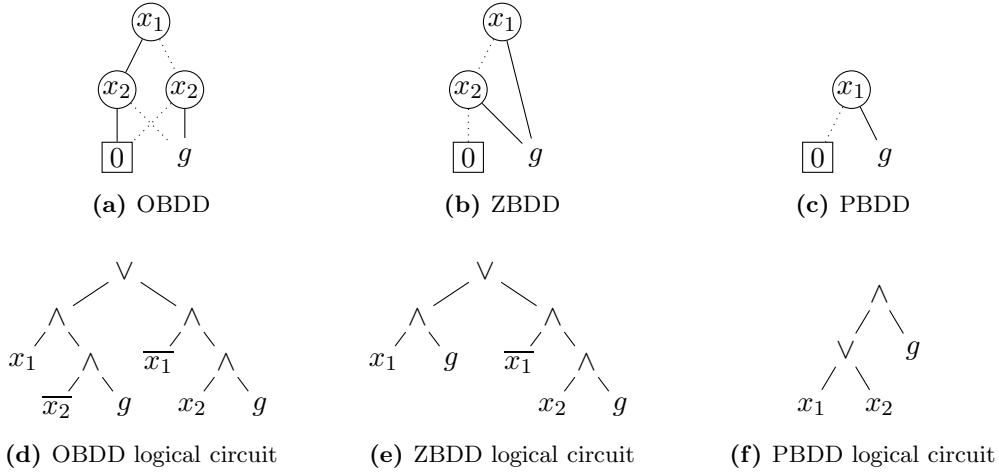
\begin{figure}[H]
    \centering
    \begin{subfigure}[t]{0.3\textwidth}
        \centering
        \begin{tikzpicture}[
                scale=0.3,
                every path/.style={>=latex},
                every node/.style={draw,circle},
                inner sep=0pt,
                minimum size=0.5cm,
                line width=1pt,
                thin,
                font=\normalsize
                ]

                \node[] (a0) at (0,0)  {$x_1$};
                \node[] (a01) at (-1.5,-3)     {$x_2$};
                \node[] (a11) at (1.5,-3)     {$x_2$};
                \node[rectangle,minimum size=0.4cm] (false) at (-1.5,-6)   {$0$};
                \node[draw=none] (true) at (1.5,-6)   {$g$};

                \draw[] (a0) edge (a01);
                \draw[dotted] (a0) edge (a11);
                \draw[dotted] (a01) edge (true);
                \draw[] (a01) edge (false);
                \draw[] (a11) edge (true);
                \draw[dotted] (a11) edge (false);

        \end{tikzpicture}
        \caption{OBDD}
    \end{subfigure}%
    \begin{subfigure}[t]{0.3\textwidth}
        \centering
        \begin{tikzpicture}[
                scale=0.3,
                every path/.style={>=latex},
                every node/.style={draw,circle},
                inner sep=0pt,
                minimum size=0.5cm,
                line width=1pt,
                thin,
                font=\normalsize
                ]

                \node[] (a0) at (0,0)  {$x_1$};
                \node[] (a01) at (-1.5,-3)     {$x_2$};
                \node[rectangle,minimum size=0.4cm] (false) at (-1.5,-6)   {$0$};
                \node[draw=none] (true) at (1.5,-6)   {$g$};

                \draw[dotted] (a0) edge (a01);
                \draw[] (a0) edge (true);
                \draw[] (a01) edge (true);
                \draw[dotted] (a01) edge (false);

        \end{tikzpicture}
        \caption{ZBDD}
    \end{subfigure}%
    \begin{subfigure}[t]{0.3\textwidth}
        \centering
        \begin{tikzpicture}[
                scale=0.3,
                every path/.style={>=latex},
                every node/.style={draw,circle},
                inner sep=0pt,
                minimum size=0.5cm,
                line width=1pt,
                thin,
                font=\normalsize
                ]

                \node[] (a0) at (0,-3)  {$x_1$};
                \node[rectangle,minimum size=0.4cm] (false) at (-1.5,-6)   {$0$};
                \node[draw=none] (true) at (1.5,-6)   {$g$};

                \draw[] (a0) edge (true);
                \draw[dotted] (a0) edge (false);

        \end{tikzpicture}
        \caption{PBDD}
    \end{subfigure}\\
    \vspace{1em}
    \begin{subfigure}[t]{0.3\textwidth}
        \centering
        \begin{tikzpicture}[
                scale=0.3,
                every path/.style={>=latex},
                every node/.style={draw,circle},
                inner sep=0pt,
                minimum size=0.5cm,
                line width=1pt,
                thin,
                font=\normalsize
                ]

                \node[draw=none] (or1) at (0,0) {$\vee$};
                \node[draw=none] (and1) at (-3,-2)  {$\wedge$};
                \node[draw=none] (and2) at (3,-2)  {$\wedge$};

                \node[draw=none] (a0) at (-4.5,-4)  {$x_1$};
                \node[draw=none] (and3) at (-1.5,-4)     {$\land$};

                \node[draw=none] (na0) at (1.5,-4)  {$\overline{x_1}$};
                \node[draw=none] (and4) at (4.5,-4)     {$\land$};

                \node[draw=none] (na1) at (-3,-6)     {$\overline{x_2}$};
                \node[draw=none] (omega0) at (0,-6)     {$g$};

                \node[draw=none] (a1) at (3,-6)     {$x_2$};
                \node[draw=none] (omega1) at (6,-6)     {$g$};

                \draw[-] (or1) edge (and1);
                \draw[-] (or1) edge (and2);
                \draw[-] (and1) edge (a0);
                \draw[-] (and2) edge (na0);
                \draw[-] (and1) edge (and3);
                \draw[-] (and2) edge (and4);

                \draw[-] (and3) edge (na1);
                \draw[-] (and3) edge (omega0);
                \draw[-] (and4) edge (a1);
                \draw[-] (and4) edge (omega1);

        \end{tikzpicture}
        \caption{OBDD logical circuit}
        \label{fig:examplepositivea}
    \end{subfigure}
    \begin{subfigure}[t]{0.3\textwidth}
        \centering
        \begin{tikzpicture}[
                scale=0.3,
                every path/.style={>=latex},
                every node/.style={draw,circle},
                inner sep=0pt,
                minimum size=0.5cm,
                line width=1pt,
                thin,
                font=\normalsize
                ]

                \node[draw=none] (or1) at (0,0) {$\vee$};
                \node[draw=none] (and1) at (-3,-2)  {$\wedge$};
                \node[draw=none] (and2) at (3,-2)  {$\wedge$};

                \node[draw=none] (a0) at (-4.5,-4)  {$x_1$};
                \node[draw=none] (and3) at (-1.5,-4)     {$g$};

                \node[draw=none] (na0) at (1.5,-4)  {$\overline{x_1}$};
                \node[draw=none] (and4) at (4.5,-4)     {$\land$};

                \node[draw=none] (a1) at (3,-6)     {$x_2$};
                \node[draw=none] (omega1) at (6,-6)     {$g$};

                \draw[-] (or1) edge (and1);
                \draw[-] (or1) edge (and2);
                \draw[-] (and1) edge (a0);
                \draw[-] (and2) edge (na0);
                \draw[-] (and1) edge (and3);
                \draw[-] (and2) edge (and4);

                \draw[-] (and4) edge (a1);
                \draw[-] (and4) edge (omega1);

        \end{tikzpicture}
        \caption{ZBDD logical circuit}
        \label{fig:examplepositivea}
    \end{subfigure}
    \begin{subfigure}[t]{0.3\textwidth}
        \centering
        \begin{tikzpicture}[
                scale=0.3,
                every path/.style={>=latex},
                every node/.style={draw,circle},
                inner sep=0pt,
                minimum size=0.5cm,
                line width=1pt,
                thin,
                font=\normalsize
                ]

                \node[draw=none] (or1) at (0,0) {$\wedge$};
                \node[draw=none] (and1) at (-1.5,-2.5)  {$\vee$};
                \node[draw=none] (and2) at (1.5,-2.5)  {$g$};
                \node[draw=none] (a1) at (-3,-5)  {$x_1$};
                \node[draw=none] (alpha) at (0,-5)     {$x_2$};

                \draw[-] (or1) edge (and1);
                \draw[-] (or1) edge (and2);
                \draw[-] (and1) edge (a1);
                \draw[-] (and1) edge (alpha);
        \end{tikzpicture}
         \caption{PBDD logical circuit}
        \label{fig:examplepositiveb}
    \end{subfigure}
    \caption{From BDDs to logical circuits}
    \label{fig:examplepositiveshannon}
\end{figure}

Figure~\ref{fig:examplepositiveshannon} shows that there are functions where the corresponding minimal PBDD and OBDD differ exponentially in size. Not represented in the figure, the collapse rule additionally removes the nodes that share the same positive cofactor with their parents (Figure~\ref{fig:collapserule}).

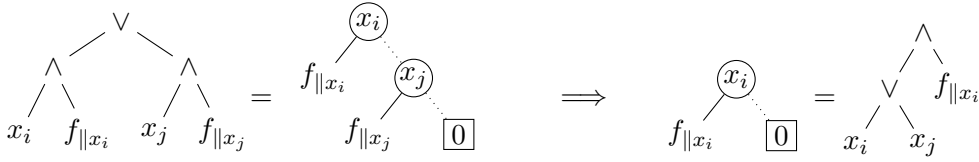
\begin{figure}[H]
    \centering
     \begin{subfigure}[t]{0.16\textwidth}
        \centering
        \begin{tikzpicture}[
                scale=0.3,
                every path/.style={>=latex},
                every node/.style={draw,circle},
                inner sep=0pt,
                minimum size=0.5cm,
                line width=1pt,
                thin,
                font=\normalsize
                ]

                \node[draw=none] (or1) at (0,0) {$\vee$};
                \node[draw=none] (and1) at (-3,-2)  {$\wedge$};
                \node[draw=none] (and2) at (3,-2)  {$\wedge$};
                \node[draw=none] (a1) at (-4.5,-5)  {$x_i$};
                \node[draw=none] (a2) at (1.5,-5)  {$x_j$};
                \node[draw=none] (alpha) at (-1.5,-5)     {$f_{{\parallel}x_i}$};
                \node[draw=none] (beta) at (4.5,-5)     {$f_{{\parallel}x_j}$};

                \draw[-] (or1) edge (and1);
                \draw[-] (or1) edge (and2);
                \draw[-] (and1) edge (a1);
                \draw[-] (and2) edge (a2);
                \draw[-] (and1) edge (alpha);
                \draw[-] (and2) edge (beta);
        \end{tikzpicture}
    \end{subfigure}
    \begin{subfigure}[t]{0.05\textwidth}
        \centering
        \begin{tikzpicture}[
                scale=0.3,
                every path/.style={>=latex},
                every node/.style={draw,circle},
                inner sep=0pt,
                minimum size=0.5cm,
                line width=1pt,
                thin,
                font=\normalsize
                ]

                \node[draw=none] () at (0,0)   {};
                \node[draw=none] (a0) at (2,2)  {$=$};

        \end{tikzpicture}
    \end{subfigure}%
    \begin{subfigure}[t]{0.22\textwidth}
        \centering
        \begin{tikzpicture}[
                scale=0.3,
                every path/.style={>=latex},
                every node/.style={draw,circle},
                inner sep=0pt,
                minimum size=0.5cm,
                line width=1pt,
                thin,
                font=\normalsize
                ]

                \node[] (x1) at (0,0)  {$x_i$};
                \node[draw=none] (fx1) at (-2,-2.5)     {};
                \node[draw=none] (fx11) at (-2,-2.5)     {$f_{{\parallel}x_i}$};
                \node[] (x2) at (2,-2.5)     {$x_j$};
                \node[draw=none] (fx2) at (0,-5)     {};
                \node[draw=none] (fx22) at (0,-5)     {$f_{{\parallel}x_j}$};
                \node[rectangle,minimum size=0.4cm] (false) at (4,-5)   {$0$};

                \draw[] (x1) edge (fx1);
                \draw[dotted] (x1) edge (x2);
                \draw[dotted] (x2) edge (false);
                \draw[] (x2) edge (fx2);

        \end{tikzpicture}
    \end{subfigure}%
    \begin{subfigure}[t]{0.12\textwidth}
        \centering
        \begin{tikzpicture}[
                scale=0.3,
                every path/.style={>=latex},
                every node/.style={draw,circle},
                inner sep=0pt,
                minimum size=0.5cm,
                line width=1pt,
                thin,
                font=\normalsize
                ]

                \node[draw=none] () at (0,0)   {};
                \node[draw=none] (a0) at (0,2)  {$\Longrightarrow$};

        \end{tikzpicture}
    \end{subfigure}%
    \begin{subfigure}[t]{0.13\textwidth}
        \centering
        \begin{tikzpicture}[
                scale=0.3,
                every path/.style={>=latex},
                every node/.style={draw,circle},
                inner sep=0pt,
                minimum size=0.5cm,
                line width=1pt,
                thin,
                font=\normalsize
                ]

                \node[] (x1) at (0,0)  {$x_i$};
                \node[draw=none] (fx1) at (-2,-2.5)     {};
                \node[draw=none] (fx11) at (-2,-2.5)     {$f_{{\parallel}x_i}$};
                \node[rectangle,minimum size=0.4cm] (false) at (2,-2.5)   {$0$};

                \draw[] (x1) edge (fx1);
                \draw[dotted] (x1) edge (false);

        \end{tikzpicture}
    \end{subfigure}%
    \begin{subfigure}[t]{0.02\textwidth}
        \centering
        \begin{tikzpicture}[
                scale=0.3,
                every path/.style={>=latex},
                every node/.style={draw,circle},
                inner sep=0pt,
                minimum size=0.5cm,
                line width=1pt,
                thin,
                font=\normalsize
                ]

                \node[draw=none] () at (0,0)   {};
                \node[draw=none] (a0) at (0,2)  {$=$};

        \end{tikzpicture}
    \end{subfigure}%
    \begin{subfigure}[t]{0.14\textwidth}
        \centering
        \begin{tikzpicture}[
                scale=0.3,
                every path/.style={>=latex},
                every node/.style={draw,circle},
                inner sep=0pt,
                minimum size=0.5cm,
                line width=1pt,
                thin,
                font=\normalsize
                ]

                \node[draw=none] (or1) at (0,0) {$\wedge$};
                \node[draw=none] (and1) at (-1.5,-2.5)  {$\vee$};
                \node[draw=none] (and2) at (1.5,-2.5)  {$f_{{\parallel}x_i}$};
                \node[draw=none] (a1) at (-3,-5)  {$x_i$};
                \node[draw=none] (alpha) at (0,-5)     {$x_j$};

                \draw[-] (or1) edge (and1);
                \draw[-] (or1) edge (and2);
                \draw[-] (and1) edge (a1);
                \draw[-] (and1) edge (alpha);
        \end{tikzpicture}
    \end{subfigure}
    \caption{Application of the collapse rule, where cofactors are equal}
    \label{fig:collapserule}
\end{figure}

The semantics of nodes whose child has been removed changes. A missing node, inferable by the support set $\mathcal{S}$ and variable order, indicates the application of the collapse rule, i.e., the distributive law on $x_i$ and $x_j$. There is no ambiguity regarding the delete rule as it can never be applied on $\mathcal{A}(X)$ due to constraint clauses $f^c$. Note that the delete rule can however be applied in case one optimizes Boolean variables of the BN by representing them with only one literal in $f^e$, as opposed to two. This will delete literals from the induced circuit and result in an inconsistent model count with regard to the probability distribution. It is precisely for this reason why the collapse rule uses the distributive law on involved literals to simplify the induced circuit, as apposed to deleting them from it. The combination of the merge and collapse rule allow for more fine grained control in exploiting CSI, because it allows independence given a subset of the values to be expressed more efficiently when dealing with multi-valued variables.

\begin{restatable}{proposition}{theoremequivalent}
\label{th:equivalent}
    An OBDD representing Boolean function $f$ and an PBDD representing $\mathcal{E}(f)$ induce isomorphic logical circuits under Boolean identity, given an appropriate ordering.
\end{restatable}

\begin{restatable}{proposition}{theoremequivalenttwo}
\label{th:equivalenttwo}
    Given an ordering on $\mathcal{A}(X)$, the size of PBDD~$\varphi$ is less than the size of OBDD~$\psi$ when both representing $\mathcal{E}(f) = f^e$, where $f$ is defined over variables $X$.
\end{restatable}

\subsection{Adding Probabilities as Weights}
Encoding $\mathcal{E}$ represents a BN as a weighted propositional formula. We extend PBDDs to \emph{weighted} PBDDs (WPBDD) using an intuitive scheme, taking advantage of the fact that probabilities are fully implied by the variables in the BN. Traditionally, an empty clause would result in a contradiction, i.e., the instantiation is unsatisfiable. We \emph{implicitly} assign weight $\omega$ of empty clause $c$ to the edge it is associated with, and remove $c$ from the expression. When multiple empty clauses are associated with an edge, we simply assign the conjunction (multiplication) of their weights to the edge. To maintain canonicity, we only assign weights on the side of the positive cofactor.

The collapse rule previously introduced can easily be extended for the non-binary and weighted case as shown in Figure~\ref{fig:weightedcollapserule}.

\begin{figure}[H]
    \centering
    \newcommand{\sepfig}{0.18}
     \begin{subfigure}[t]{\sepfig\textwidth}
        \centering
        \begin{tikzpicture}[
                scale=0.3,
                every path/.style={>=latex},
                every node/.style={draw,circle},
                inner sep=0pt,
                minimum size=0.5cm,
                line width=1pt,
                thin,
                font=\normalsize
                ]

                \node[draw=none] (or1) at (0.5,0) {$\vee$};
                \node[draw=none] (and1) at (-2,-2)  {$\wedge$};
                \node[draw=none] (and2) at (3,-2)  {$\wedge$};

                \node[draw=none] (a1) at (-3,-4.5)  {$x_i$};
                \node[draw=none] (beta) at (4,-4.5)     {$\wedge$};

                \node[draw=none] (a2) at (2,-4.5)  {$x_j$};
                \node[draw=none] (alpha) at (-1,-4.5)     {$\wedge$};

                \node[draw=none] (w2) at (-2,-7)     {$\omega_i$};
                \node[draw=none] (g2) at (0,-7)     {\ $f_{{\parallel}x_i}$};

                \node[draw=none] (w1) at (3,-7)     {$\omega_j$};
                \node[draw=none] (g1) at (5,-7)     {\ $f_{{\parallel}x_j}$};

                \draw[-] (or1) edge (and1);
                \draw[-] (or1) edge (and2);
                \draw[-] (and1) edge (a1);
                \draw[-] (and2) edge (a2);
                \draw[-] (and1) edge (alpha);
                \draw[-] (and2) edge (beta);

                \draw[-] (beta) edge (w1);
                \draw[-] (beta) edge (g1);

                \draw[-] (alpha) edge (w2);
                \draw[-] (alpha) edge (g2);
        \end{tikzpicture}
    \end{subfigure}
    \begin{subfigure}[t]{0.05\textwidth}
        \centering
        \begin{tikzpicture}[
                scale=0.3,
                every path/.style={>=latex},
                every node/.style={draw,circle},
                inner sep=0pt,
                minimum size=0.5cm,
                line width=1pt,
                thin,
                font=\normalsize
                ]

                \node[draw=none] () at (0,0)   {};
                \node[draw=none] (a0) at (0,2)  {$=$};

        \end{tikzpicture}
    \end{subfigure}%
    \begin{subfigure}[t]{0.15\textwidth}
        \centering
        \begin{tikzpicture}[
                scale=0.3,
                every path/.style={>=latex},
                every node/.style={draw,circle},
                inner sep=0pt,
                minimum size=0.5cm,
                line width=1pt,
                thin,
                font=\normalsize
                ]

                \node[] (x1) at (0,0)  {$x_i$};
                \node[draw=none] (w1) at (-1.5,-1.5)  {$\omega_i$};
                \node[draw=none] (fx1) at (-1.5,-3.5)     {};
                \node[draw=none] (fx11) at (-1.5,-3.5)     {$f_{{\parallel}x_i}$};

                \node[] (x2) at (1.5,-3.5)     {$x_j$};
                \node[draw=none] (fx2) at (0,-7)     {};
                \node[draw=none] (fx22) at (0,-7)     {$f_{{\parallel}x_j}$};
                \node[draw=none] (w1) at (0,-5)  {$\omega_j$};
                \node[rectangle,minimum size=0.4cm] (false) at (3,-7)   {$0$};

                \draw[] (x1) edge (fx1);
                \draw[dotted] (x1) edge (x2);
                \draw[dotted] (x2) edge (false);
                \draw[] (x2) edge (fx2);

        \end{tikzpicture}
    \end{subfigure}%
    \begin{subfigure}[t]{0.12\textwidth}
        \centering
        \begin{tikzpicture}[
                scale=0.3,
                every path/.style={>=latex},
                every node/.style={draw,circle},
                inner sep=0pt,
                minimum size=0.5cm,
                line width=1pt,
                thin,
                font=\normalsize
                ]

                \node[draw=none] () at (0,0)   {};
                \node[draw=none] (a0) at (0,2)  {$\Longrightarrow$};

        \end{tikzpicture}
    \end{subfigure}%
    \begin{subfigure}[t]{0.13\textwidth}
        \centering
        \begin{tikzpicture}[
                scale=0.3,
                every path/.style={>=latex},
                every node/.style={draw,circle},
                inner sep=0pt,
                minimum size=0.5cm,
                line width=1pt,
                thin,
                font=\normalsize
                ]

                \node[] (x1) at (0,0)  {$x_i$};
                \node[draw=none] (fx1) at (-2,-3.5)     {};
                \node[draw=none] (fx11) at (-2,-3.5)     {$f_{{\parallel}x_i}$};
                \node[rectangle,minimum size=0.4cm] (false) at (2,-3.5)   {$0$};
                \node[draw=none] (w) at (-2,-1.5)     {$\omega_i$};
                \draw[] (x1) edge (fx1);
                \draw[dotted] (x1) edge (false);

        \end{tikzpicture}
    \end{subfigure}%
    \begin{subfigure}[t]{0.02\textwidth}
        \centering
        \begin{tikzpicture}[
                scale=0.3,
                every path/.style={>=latex},
                every node/.style={draw,circle},
                inner sep=0pt,
                minimum size=0.5cm,
                line width=1pt,
                thin,
                font=\normalsize
                ]

                \node[draw=none] () at (0,0)   {};
                \node[draw=none] (a0) at (0,2)  {$=$};

        \end{tikzpicture}
    \end{subfigure}%
    \begin{subfigure}[t]{0.20\textwidth}
        \centering
        \begin{tikzpicture}[
                scale=0.3,
                every path/.style={>=latex},
                every node/.style={draw,circle},
                inner sep=0pt,
                minimum size=0.5cm,
                line width=1pt,
                thin,
                font=\normalsize
                ]

                \node[draw=none] (or1) at (0.5,0) {$\wedge$};
                \node[draw=none] (and1) at (-2,-2)  {$\vee$};
                \node[draw=none] (and2) at (3,-2)  {$\wedge$};

                \node[draw=none] (a2) at (2,-4.5)  {$\omega_i$};
                \node[draw=none] (alpha) at (-1,-4.5)     {$x_j$};

                \node[draw=none] (a1) at (-3,-4.5)  {$x_i$};
                \node[draw=none] (beta) at (4,-4.5)     {$f_{{\parallel}x_i}$};

                \draw[-] (or1) edge (and1);
                \draw[-] (or1) edge (and2);
                \draw[-] (and1) edge (a1);
                \draw[-] (and2) edge (a2);
                \draw[-] (and1) edge (alpha);
                \draw[-] (and2) edge (beta);
        \end{tikzpicture}
    \end{subfigure}
    \caption{Application of the collapse rule, where functions $f_{{\parallel}x_i} = f_{{\parallel}x_j}$ and weights $\omega_i = \omega_j$ }
    \label{fig:weightedcollapserule}
\end{figure}
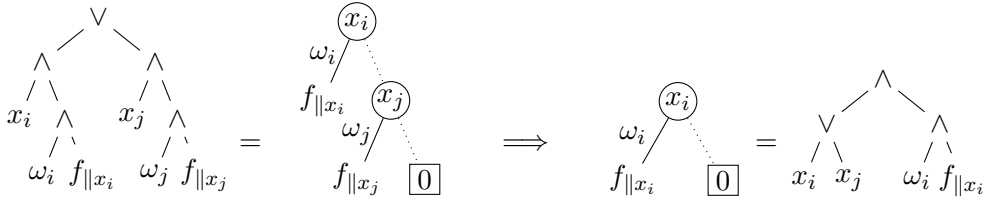

\begin{definition}\label{def:wpbdd}
    A \emph{weighted} PBDD (WPBDD) representing Boolean function $\mathcal{E}(f) = f^e$, where $f$ is defined over variables $X$, is a PBDD where each node $v$ is a tuple ${\langle}x_i,W,f^e_{{\parallel}x_i},f^e_{|\overline{x_i}}\rangle$ that represents a weighted reduced positive Shannon expansion:

\[ f^e \ \ \models\ \  x_i \land (W \land f^e_{{\parallel}x_i})\ \ \lor\ \ f_{|\overline{x_i}}, \]

\noindent where $x_i \in \mathcal{A}(X)$, and $W$ is a conjunction of weights $\omega_i$ that correspond to $f^e_{{\parallel}x_i}$. Positive and negative cofactors are as described by Lemma~\ref{def:positiveshannon}. It is a \emph{canonical} representation if \emph{reduced} by applying the following rules:

    \begin{enumerate}
        \item Merge rule: All isomorphic subgraphs are merged.
        \item Collapse rule: remove direct descendant $u$ of node $v$ iff $W \land f_{{\parallel}x_i} = W \land f_{{\parallel}x_j}$, where $var(v) = x_i$ and $var(u) = x_j$, with $x_i,x_j \in \mathcal{A}(x)$ and $x \in X$.
    \end{enumerate}
\end{definition}

\begin{example}\label{ex:partitionedcpt}
    Consider the CPTs from Example~\ref{ex:full}, where per CPT, equal probabilities are represented by unique symbolic weights $\omega_i$.

\begin{figure}[H]
    \centering
    \begin{tikzpicture}[
        scale=0.5,
        thin,
        font=\normalsize
        ]
        \node[draw=none](tab) at (0,0){
            \begin{minipage}{0.15\textwidth}
                \centering

                    \setlength{\tabcolsep}{4pt}
                    \begin{tabular}[t]{c | c}
                        $P(a_1)$ & $P(a_2)$ \\\hline
                        &\\[-2ex]
                        $\omega_1$ & $\omega_1$\\
                    \end{tabular}
            \end{minipage}
            \begin{minipage}{0.30\textwidth}
                \centering
                    \begin{tabular}[t]{c || c | c | c }
                        \setlength{\tabcolsep}{2pt}
                        $a$ & $P(b_1 | a)$ & $P(b_2 | a)$ & $P(b_3 | a)$\\\hline
                        &&&\\[-2ex]
                        \small{1} & $\omega_2$ & $\omega_2$ & $\omega_3$\\
                        \small{2} & $\omega_2$ & $\omega_2$ & $\omega_3$\\
                    \end{tabular}
            \end{minipage}
        };

        \begin{scope}[shift={(0.3,-0.35)}]
            \draw [rounded corners,dashed] (-6.8,0.1)--(-6.8,-0.8)--(-3.1,-0.8)--(-3.1,0.1)--cycle;
        \end{scope}

        \begin{scope}[shift={(0.6,-0.3)}]
            \draw [rounded corners,dashed] (0,0.5)--(0,-1.3)--(4.5,-1.3)--(4.5,0.5)--cycle;
        \end{scope}

        \begin{scope}[shift={(1.8,-0.2)}]
            \draw [rounded corners,dashed] (6.6,0.4)--(6.6,-1.4)--(5.5,-1.4)--(5.5,0.4)--cycle;
        \end{scope}
    \end{tikzpicture}
\end{figure}

Figure~\ref{fig:reduced} shows the minimization of a WPBDD using variable ordering $a_1 < a_2 < b_1 < b_2 < b_3$, that represents the BN with 3 probabilities instead of 8. Figure~\ref{fig:comparison} shows the comparison of this WPBDD with an OBDD representing the same function, given variable ordering $a_1 < a_2 < \omega_1 < b_1 < b_2 < b_3 < \omega_2 < \omega_3$, which results in a minimal OBDD that obeys the partial ordering used for the WPBDD.

\end{example}

\input{figures/reduced.tex}

\begin{figure}[t]
    \centering
    \begin{subfigure}[t]{0.3\textwidth}
        \centering
        \begin{tikzpicture}[
                scale=0.3,
                every path/.style={>=latex},
                every node/.style={draw,circle},
                inner sep=0pt,
                minimum size=0.5cm,
                line width=1pt,
                thin,
                font=\normalsize
                ]

                \node[] (a0) at (-1.5,-6) {$a_1$};
                \node[] (b0) at (4.5,-9)  {$b_1$};
                \node[] (b2) at (1.5,-12)   {$b_3$};

                \node[draw=none] (w0) at (2,-6.5)      {$\omega_1$};
                \node[draw=none] (w1) at (5.5,-12.5) {$\omega_2$};

                \node[draw=none] (w2) at (2.25,-14.3)   {$\omega_3$};
                \node[rectangle,minimum size=0.4cm] (false) at (-1.5,-16) {$0$};
                \node[rectangle,minimum size=0.4cm] (true) at (4.5,-16)   {$1$};

                \draw[dotted] (a0) edge (false);
                \draw[dotted] (b0) edge (b2);
                \draw[dotted] (b2) edge (false);
                \draw[] (b0) edge (true);
                \draw[] (b2) edge (true);

                \draw[] (a0) edge (b0);
        \end{tikzpicture}
        \caption{WPBDD}
    \end{subfigure}
    \vspace{1em}
    \newcommand{\sepfig}{0.3}
    \begin{subfigure}[t]{\sepfig\textwidth}
        \centering
        \begin{tikzpicture}[
                scale=0.25,
                every path/.style={>=latex},
                every node/.style={draw,circle},
                inner sep=0pt,
                minimum size=0.5cm,
                line width=1pt,
                thin,
                font=\normalsize
                ]

                \node[] (a1) at (0,0) {$a_1$};
                \node[] (a21) at (-3,-3) {$a_2$};
                \node[] (a22) at (3,-3) {$a_2$};
                \node[] (w1) at (0,-6) {$\omega_1$};
                \node[] (b1) at (3,-9) {$b_1$};
                \node[] (b21) at (0,-12) {$b_2$};
                \node[] (b22) at (6,-12) {$b_2$};
                \node[] (b31) at (3,-15) {$b_3$};
                \node[] (b32) at (9,-15) {$b_3$};
                \node[] (w2) at (6,-18) {$\omega_2$};
                \node[] (w3) at (12,-18) {$\omega_3$};

                \node[rectangle,minimum size=0.4cm] (false) at (-3,-21) {$0$};
                \node[rectangle,minimum size=0.4cm] (true) at (9,-21)   {$1$};
                \draw[] (a1) edge (a21);
                \draw[dotted] (a1) edge (a22);
                \draw[] (a21) edge (false);
                \draw[dotted] (a21) edge (w1);
                \draw[dotted] (a22) edge (false);
                \draw[] (a22) edge (w1);
                \draw[] (w1) edge (b1);
                \draw[dotted] (w1) edge (false);
                \draw[] (b1) edge (b21);
                \draw[dotted] (b1) edge (b22);
                \draw[] (b21) edge (false);
                \draw[dotted] (b21) edge (b31);
                \draw[] (b22) edge (b31);
                \draw[dotted] (b22) edge (b32);
                \draw[] (b31) edge (false);
                \draw[dotted] (b31) edge (w2);
                \draw[] (b32) edge (w3);
                \draw[dotted] (b32) edge (false);

                \draw[] (w2) edge (true);
                \draw[dotted] (w2) edge (false);
                \draw[] (w3) edge (true);
                \draw[dotted] (w3) edge (false);
        \end{tikzpicture}
        \caption{OBDD}
    \end{subfigure}
    \vspace{-1em}
    \caption{WPBDD and OBDD comparison}
    \label{fig:comparison}
\end{figure}
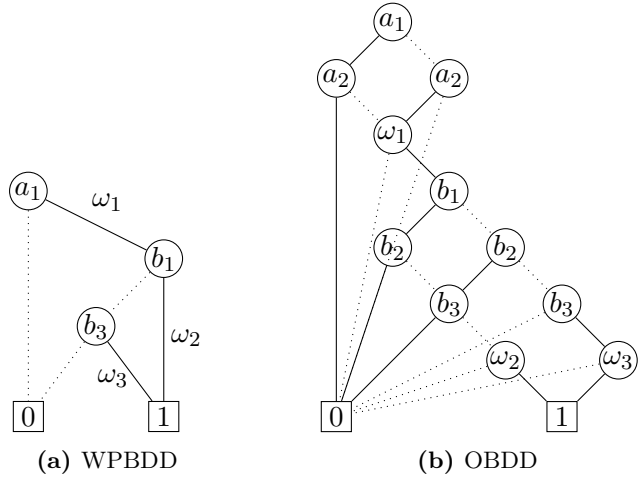

\section{Symbolic Inference}\label{sec:symbolicinference}
We perform Bayesian inference through a three phase process: encoding, compiling and model counting:

\begin{center}
    \begin{footnotesize}
\begin{tabular}{l | l | l | l}
    & Composition & Input & Output\\\hline&&&\\[-2ex]
    Encoder   & Encoding $\mathcal{E}$     & $f$                        & Theory $\mathcal{T}$ + $f^e$\\
    Compiler  & $\mathcal{T}$-solver + SAT & $\mathcal{T}$ + $f^e$ & WPBDD $\varphi$\\
    Counter   & $\mathcal{T}$-solver + WMC & $\mathcal{T}$ + $\varphi$  & $P(x|\boldsymbol{e})$\\
\end{tabular}
\end{footnotesize}
\end{center}

A BN represented by $f$ is first encoded by the \emph{encoder} as Boolean function $f^e$ using encoding $\mathcal{E}$ as defined in Section~\ref{sec:encoding}. The encoder also provides background theory~$\mathcal{T}$ representing $f^c$, i.e., the constraints among variables that support mapping $\mathcal{M}$.

The \emph{compiler} uses a lazy SMT-solver to record evaluation paths and as a WPBDD, given $f^e$ and theory $\mathcal{T}$. A lazy SMT-solver combines a SAT-solver with a theory-solver (or $\mathcal{T}$-solver) for some theory~$\mathcal{T}$. Traditionally, the role of a theory-solver is to purely report back on the satisfiability of $\mathcal{T}$. We have extended the theory-solver to provide more information in order to support implicit conditioning in the SAT-solver.

The \emph{counter} computes the probability of $x$ given evidence $\boldsymbol{e}$, by translating the provided WPBDD into an arithmetic circuit and using the extended capabilities of the $T$-solver to properly instantiate the variables.

\subsection{Compilation}\label{sec:compilation}
The order of decomposition greatly influences representation size. Compilation therefore reduces to finding the optimal variable ordering. Satisfiability (SAT) is key during compilation. Normally, when CNF $f$ contains an empty clause we derive a contradiction. Note that if all contradicting clauses are weighted, we supersede the contradiction and introduce their weights into our representation at the corresponding edge. In order to obtain a minimal BDD representation, the search space of all variable instantiations is traversed with a DPLL-style algorithm, in order to find partial instantiations that describe equal (sub)functions \cite{bacchus2003dpll}, as depicted in Figure~\ref{fig:compilation}.

\begin{figure}[H]
    \centering
    \begin{subfigure}[t]{0.15\textwidth}
        \centering
        \begin{tikzpicture}[
                scale=0.3,
                every path/.style={>=latex},
                every node/.style={draw,circle},
                inner sep=0pt,
                minimum size=0.5cm,
                line width=1pt,
                thin,
                font=\normalsize
                ]

                \node[] (a) at (0,0)  {$a$};
                \node[] (b1) at (-1.5,-3)     {$b$};
                \node[draw=none] (b2) at (1.5,-3)     {};
                \node[] (c1) at (-3,-6)   {$c$};
                \node[draw=none] (c2) at (0,-6)   {};
                \node[rectangle,minimum size=0.4cm] (false) at (-1.5,-9)   {$0$};
                \node[rectangle,minimum size=0.4cm] (true) at (1.5,-9)   {$1$};
                \draw[dotted] (a) edge (b1);
                \draw[dotted] (b1) edge (c1);
                \draw[dotted] (c1) edge (true);

        \end{tikzpicture}
        \caption{$\overline{a} - \overline{b} - \overline{c}$}
    \end{subfigure}%
    \begin{subfigure}[t]{0.15\textwidth}
        \centering
        \begin{tikzpicture}[
                scale=0.3,
                every path/.style={>=latex},
                every node/.style={draw,circle},
                inner sep=0pt,
                minimum size=0.5cm,
                line width=1pt,
                thin,
                font=\normalsize
                ]

                \node[] (a) at (0,0)  {$a$};
                \node[] (b1) at (-1.5,-3)     {$b$};
                \node[draw=none] (b2) at (1.5,-3)     {};
                \node[] (c1) at (-3,-6)   {$c$};
                \node[draw=none] (c2) at (0,-6)   {};
                \node[rectangle,minimum size=0.4cm] (false) at (-1.5,-9)   {$0$};
                \node[rectangle,minimum size=0.4cm] (true) at (1.5,-9)   {$1$};
                \draw[dotted] (a) edge (b1);
                \draw[dotted] (b1) edge (c1);
                \draw[dotted] (c1) edge (true);
                \draw[-] (c1) edge (false);
        \end{tikzpicture}
        \caption{$\overline{a} - \overline{b} - c$}
    \end{subfigure}%
    \begin{subfigure}[t]{0.05\textwidth}
        \centering
        \begin{tikzpicture}[
                scale=0.3,
                every path/.style={>=latex},
                every node/.style={draw,circle},
                inner sep=0pt,
                minimum size=0.5cm,
                line width=1pt,
                thin,
                font=\normalsize
                ]
                \node[draw=none] (dummy1) at (0,0) {};
                \node[draw=none] (dot) at (0,-2) {$\dotsc$};
                \node[draw=none] (dummy1) at (0,-5) {};
        \end{tikzpicture}
    \end{subfigure}%
    \begin{subfigure}[t]{0.15\textwidth}
        \centering
        \begin{tikzpicture}[
                scale=0.3,
                every path/.style={>=latex},
                every node/.style={draw,circle},
                inner sep=0pt,
                minimum size=0.5cm,
                line width=1pt,
                thin,
                font=\normalsize
                ]

                \node[] (a) at (0,0)  {$a$};
                \node[] (b1) at (-1.5,-3)     {$b$};
                \node[draw=none] (b2) at (1.5,-3)     {};
                \node[] (c1) at (-3,-6)   {$c$};
                \node[] (c2) at (0,-6)   {$c$};
                \node[rectangle,minimum size=0.4cm] (false) at (-1.5,-9)   {$0$};
                \node[rectangle,minimum size=0.4cm] (true) at (1.5,-9)   {$1$};
                \draw[dotted] (a) edge (b1);

                \draw[dotted] (b1) edge (c1);
                \draw[-] (b1) edge (c2);

                \draw[dotted] (c1) edge (true);
                \draw[-] (c1) edge (false);

                \draw[dotted] (c2) edge (false);
                \draw[-] (c2) edge (true);

        \end{tikzpicture}
        \caption{$\overline{a} - b - c$}
    \end{subfigure}%
    \begin{subfigure}[t]{0.15\textwidth}
        \centering
        \begin{tikzpicture}[
                scale=0.3,
                every path/.style={>=latex},
                every node/.style={draw,circle},
                inner sep=0pt,
                minimum size=0.5cm,
                line width=1pt,
                thin,
                font=\normalsize
                ]

                \node[] (a) at (0,0)  {$a$};
                \node[] (b1) at (-1.5,-3)     {$b$};
                \node[] (b2) at (1.5,-3)     {$b$};
                \node[] (c1) at (-3,-6)   {$c$};
                \node[] (c2) at (0,-6)   {$c$};
                \node[rectangle,minimum size=0.4cm] (false) at (-1.5,-9)   {$0$};
                \node[rectangle,minimum size=0.4cm] (true) at (1.5,-9)   {$1$};
                \draw[dotted] (a) edge (b1);
                \draw[-] (a) edge (b2);

                \draw[dotted] (b1) edge (c1);
                \draw[-] (b1) edge (c2);

                \draw[dotted] (b2) edge (c2);

                \draw[dotted] (c1) edge (true);
                \draw[-] (c1) edge (false);

                \draw[dotted] (c2) edge (false);
                \draw[-] (c2) edge (true);

        \end{tikzpicture}
        \caption{$a - \overline{b}$}
        \label{fig:compilationd}
    \end{subfigure}
    \begin{subfigure}[t]{0.05\textwidth}
        \centering
        \begin{tikzpicture}[
                scale=0.3,
                every path/.style={>=latex},
                every node/.style={draw,circle},
                inner sep=0pt,
                minimum size=0.5cm,
                line width=1pt,
                thin,
                font=\normalsize
                ]
                \node[draw=none] (dummy1) at (0,0) {};
                \node[draw=none] (dot) at (-0.5,-2) {$\dotsc$};
                \node[draw=none] (dummy1) at (0,-5) {};
        \end{tikzpicture}
    \end{subfigure}%
    \begin{subfigure}[t]{0.15\textwidth}
        \centering
        \begin{tikzpicture}[
                scale=0.3,
                every path/.style={>=latex},
                every node/.style={draw,circle},
                inner sep=0pt,
                minimum size=0.5cm,
                line width=1pt,
                thin,
                font=\normalsize
                ]

                \node[] (a) at (0,0)  {$a$};
                \node[] (b1) at (-1.5,-3)     {$b$};
                \node[] (b2) at (1.5,-3)     {$b$};
                \node[] (c1) at (-3,-6)   {$c$};
                \node[] (c2) at (0,-6)   {$c$};
                \node[rectangle,minimum size=0.4cm] (false) at (-1.5,-9)   {$0$};
                \node[rectangle,minimum size=0.4cm] (true) at (1.5,-9)   {$1$};
                \draw[dotted] (a) edge (b1);
                \draw[-] (a) edge (b2);

                \draw[dotted] (b1) edge (c1);
                \draw[-] (b1) edge (c2);

                \draw[-] (b2) edge (true);
                \draw[dotted] (b2) edge (c2);

                \draw[dotted] (c1) edge (true);
                \draw[-] (c1) edge (false);

                \draw[dotted] (c2) edge (false);
                \draw[-] (c2) edge (true);

        \end{tikzpicture}
        \caption{$a - b$}
    \end{subfigure}
    \caption{Compilation example for non-trivial function, given ordering $a \rightarrow b \rightarrow c$}
    \label{fig:compilation}
\end{figure}
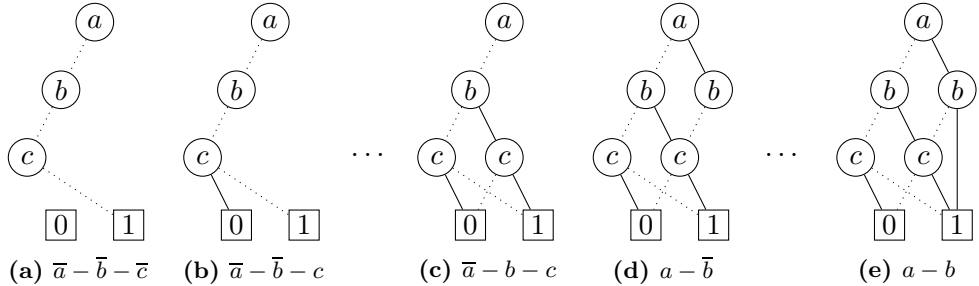

Algorithm~\ref{alg:compiler} shows a Depth-first search/dynamic programming (DFS+DP) approach that uses a lazy SMT-solver ({\small\ttfamily\selectfont solver}) to compile $f^e$ into a WPBDD, given some ordering on the variables ({\small\ttfamily\selectfont ordering}), where $\top$ and $\bot$ denote the terminal nodes representing \emph{true} and \emph{false}, respectively.

The lazy SMT-solver is unique in the sense that there exists a link beyond satisfiability feedback between the $\mathcal{T}$-solver and the SAT-solver ({\small\ttfamily\selectfont solver.theory} and {\small\ttfamily\selectfont solver.sat}, respectively). The traditional role of the theory-solver is to solely report back on satisfiability. To implement implicit conditioning, we have extended it to also provide its unit clauses, that are used by the SAT-solver to further condition $f^e$. This connection is possible because $f^e$ and theory~$\mathcal{T}$ both essentially depend on the same variables. The SMT-solver reports $f^e$ to be  satisfiable only when both the theory and SAT-solver agree on this. Storing intermediate states as an undo mechanism for the solver is infeasible, thus it has the ability to dynamically undo any conditioning ({\small\ttfamily\selectfont solver.undo}).

The compiler uses the satisfiability state of the SMT-solver ({\small\ttfamily\selectfont solver.state}) to build the WPBDD and achieves a canonical form by applying to each subfunction the merge rule (as describe by \cite{brace1991efficient}) and collapse rule ({\small\ttfamily\selectfont apply\_merge\_rule(n)} and {\small\ttfamily\selectfont apply\_collapse\_rule(n)}, respectively).

\begin{algorithm}[t!]

\hspace{1em}\begin{minipage}[t]{.55\textwidth}
\begin{lstlisting}[mathescape=true]{topdown}
struct node {
    literal l;
    node *t, *e;
    set W;
};

enum satisfy_t {
    satisfiable   = 0,
    unsatisfied   = 1,
    unsatisfiable = 3
};

satisfy_t solver::condition(literal l){
    solver.theory.condition(l);
    solver.sat.condition(l);
    solver.sat.condition(
        solver.theory.unit_clauses());

    return solver.theory.state() |
        solver.sat.state();
}

node* condition(literal l)
    solver.condition(l);
    if(not negated(l))
        n->W += solver.sat.weights();

    node *n;
    switch(solver.state()){
        case unsatisfied:
            n = compile(new node,i+1);
            break;
\end{lstlisting}
\end{minipage}
\begin{minipage}[t]{.4\textwidth}
\begin{lstlisting}[mathescape=true]{topdown}
        case satisfiable:
            n = $\top$;
            break;
        case unsatisfiable:
            n = $\bot$;
            break;
    }
    solver.undo();

    return n;
}

node* compile(n,i=0){
    n->l = ordering[i];
    n->t = condition(n->l);
    n->e = condition(not n->l);

    apply_collapse_rule(n);
    apply_merge_rule(n);

    return n;
}

wpbdd* compiler($\mathcal{T}$,$f^e$){
    solver.theory.init($\mathcal{T}$);
    solver.sat.init($f^m$);

    return compile(new node);
}
\end{lstlisting}
\end{minipage}

\caption{Compiler}
\label{alg:compiler}
\end{algorithm}

\subsection{Inference by Weighted Model Counting}\label{sec:inference}
In order to perform inference by WMC, a WPBDD must be converted into a logical (refactored) form using Definition~\ref{def:wpbdd}. Recall that a missing variable along a path implies the use of the distributive law, identifiable by using the variable ordering and support set $\mathcal{S}$. The logical form can easily be translated into an arithmetic circuit according to Table~\ref{tab:convert2ac}. Note that $x \lor y$ reduces to $x + y$, when $x$ and $y$ originate from the same dimension, i.e., $x,y \in \mathcal{A}(z)$, with $z \in X$.

\begin{table}[H]
    \centering
    \begin{tabular}{c | c}
        Logical & Arithmetic\\\hline
        $x$ & $x$\\
        $\overline{x}$ & $(1-x)$\\
        $x \wedge y$ & $x * y$\\
        $x \vee y$ & $x + y - x * y$\\
    \end{tabular}
    \caption{Converting logical to arithmetic operator}
    \label{tab:convert2ac}
\end{table}

One of the reasons for using the positive Shannon decomposition is to prevent constraints among variables to be represented twice in the described process of symbolic inference: once as part of the compiled representation, and again when we substitute literals with their appropriate weight in order to perform model counting. During this later phase, theory~$\mathcal{T}$ is used to prohibit inconsistent network instantiations, preventing a state where multiple values are assigned to one variable. To perform inference, all weights $\omega_i$ are set to the probability they represent, and all other literals are set to 1. By conditioning $\mathcal{T}$ on the evidence using the theory-solver, literals are found  that conflict with the evidence in the form of unit clauses. These must be set to 0.

\begin{example}\label{ex:refactoring}
        Let $f^c = (a_1 \lor a_2)\ \land\ (\overline{a_1} \lor \overline{a_2})$ represent the constraint clauses for variable $a$ of Example~\ref{ex:full}. When computing $P(a_1)$ we condition $f^c$ on evidence $a_1$ yielding $f^c_{a_1} = \overline{a_2}$, thus evidence $a_1$ implies $a_1 = 1$ ($true$) and $a_2 = 0$ ($false$). This process is shown in Figure~\ref{fig:inference}.

\begin{figure}[H]
    \centering

    \newcommand{\sepfig}{0.3}
    \begin{subfigure}[t]{0.2\textwidth}
        \centering
        \begin{tikzpicture}[
                scale=0.3,
                every path/.style={>=latex},
                every node/.style={draw,circle},
                inner sep=0pt,
                minimum size=0.5cm,
                line width=1pt,
                thin,
                font=\normalsize
                ]

                \node[] (a0) at (-1.5,-6) {$a_1$};
                \node[] (b0) at (4.5,-9)  {$b_1$};
                \node[] (b2) at (1.5,-12)   {$b_3$};

                \node[draw=none] (w0) at (2,-6.5)      {$\omega_1$};
                \node[draw=none] (w1) at (5.5,-12.5) {$\omega_2$};

                \node[draw=none] (w2) at (2.25,-14.3)   {$\omega_3$};
                \node[rectangle,minimum size=0.4cm] (false) at (-1.5,-16) {$0$};
                \node[rectangle,minimum size=0.4cm] (true) at (4.5,-16)   {$1$};

                \draw[dotted] (a0) edge (false);
                \draw[dotted] (b0) edge (b2);
                \draw[dotted] (b2) edge (false);
                \draw[] (b0) edge (true);
                \draw[] (b2) edge (true);

                \draw[] (a0) edge (b0);
        \end{tikzpicture}
        \caption{WPBDD}
    \end{subfigure}
    \begin{subfigure}[t]{\sepfig\textwidth}
        \centering
        \begin{tikzpicture}[
                scale=0.3,
                every path/.style={>=latex},
                every node/.style={draw,circle},
                inner sep=0pt,
                minimum size=0.5cm,
                line width=1pt,
                thin,
                font=\normalsize
                ]

                \node[draw=none] (l1) at (0,0)     {$\land$};
                \node[draw=none] (l2) at (-3,-3)     {$\lor$};
                \node[draw=none] (l3) at (3,-3)     {$\land$};
                \node[draw=none] (l4) at (4.5,-6)     {$\lor$};
                \node[draw=none] (l5) at (1.5,-9)     {$\land$};
                \node[draw=none] (l6) at (7.5,-9)     {$\land$};
                \node[draw=none] (l7) at (0,-12)     {$\lor$};

                \node[draw=none] (a0) at (-4.5,-6)     {$a_1$};
                \node[draw=none] (a1) at (-1.5,-6)     {$a_2$};

                \node[draw=none] (b0) at (-1.5,-15)     {$b_1$};
                \node[draw=none] (b1) at (1.5,-15)     {$b_2$};
                \node[draw=none] (b2) at (6,-12)     {$b_3$};

                \node[draw=none] (t0) at (1.5,-6)     {$\omega_1$};

                \node[draw=none] (t1) at (3,-12)     {$\omega_2$};
                \node[draw=none] (t2) at (9,-12)     {$\omega_3$};

                \draw[] (l1) edge (l2);
                \draw[] (l1) edge (l3);
                \draw[] (l1) edge (l3);
                \draw[] (l3) edge (l4);
                \draw[] (l4) edge (l5);
                \draw[] (l4) edge (l6);
                \draw[] (l5) edge (l7);

                \draw[] (l2) edge (a0);
                \draw[] (l2) edge (a1);
                \draw[] (l7) edge (b0);
                \draw[] (l7) edge (b1);
                \draw[] (l6) edge (b2);

                \draw[] (l3) edge (t0);
                \draw[] (l5) edge (t1);
                \draw[] (l6) edge (t2);
        \end{tikzpicture}
        \caption{Logical circuit}
    \end{subfigure}
    \begin{subfigure}[t]{\sepfig\textwidth}
        \centering
        \begin{tikzpicture}[
                scale=0.3,
                every path/.style={>=latex},
                every node/.style={draw,circle},
                inner sep=0pt,
                minimum size=0.5cm,
                line width=1pt,
                thin,
                font=\normalsize
                ]

                \node[draw=none] (l1) at (0,0)     {$*$};
                \node[draw=none] (l2) at (-3,-3)     {$+$};
                \node[draw=none] (l3) at (3,-3)     {$*$};
                \node[draw=none] (l4) at (4.5,-6)     {$+$};
                \node[draw=none] (l5) at (1.5,-9)     {$*$};
                \node[draw=none] (l6) at (7.5,-9)     {$*$};
                \node[draw=none] (l7) at (0,-12)     {$+$};

                \node[draw=none] (a0) at (-4.5,-6)     {$1$};
                \node[draw=none] (a1) at (-1.5,-6)     {$0$};

                \node[draw=none] (b0) at (-1.5,-15)     {$1$};
                \node[draw=none] (b1) at (1.5,-15)     {$1$};
                \node[draw=none] (b2) at (6,-12)     {$1$};

                \node[draw=none] (t0) at (1.5,-6)     {$0.5$};

                \node[draw=none] (t1) at (3,-12)     {$0.2$};
                \node[draw=none] (t2) at (9,-12)     {$0.6$};

                \draw[] (l1) edge (l2);
                \draw[] (l1) edge (l3);
                \draw[] (l1) edge (l3);
                \draw[] (l3) edge (l4);
                \draw[] (l4) edge (l5);
                \draw[] (l4) edge (l6);
                \draw[] (l5) edge (l7);

                \draw[] (l2) edge (a0);
                \draw[] (l2) edge (a1);
                \draw[] (l7) edge (b0);
                \draw[] (l7) edge (b1);
                \draw[] (l6) edge (b2);

                \draw[] (l3) edge (t0);
                \draw[] (l5) edge (t1);
                \draw[] (l6) edge (t2);
        \end{tikzpicture}
        \caption{Instantiated arithmetic circuit}
    \end{subfigure}
    \caption{Probabilistic inference}
    \label{fig:inference}
\end{figure}
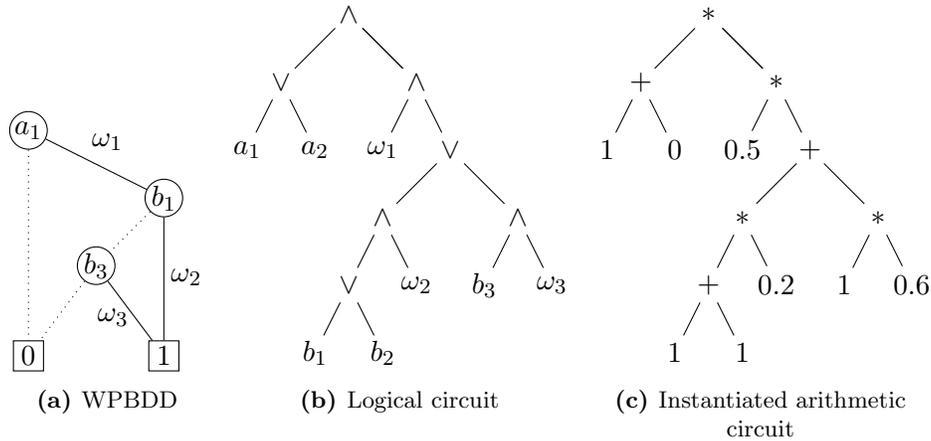

\end{example}

\section{Optimizations}\label{sec:optimizations}

\subsection{Encoding}
The constraint clauses $f^c$ of encoded function $\mathcal{E}(f) = f^e$, where $f$ is defined over variables $X$, introduce predictable symmetries into the encoding (demonstrated by Example~\ref{ex:unitclauses}). By incorporating these constraints directly into the compilation process through theory~$\mathcal{T}$, the constraint clauses $f^c$ generated by Equation~\ref{eq:variables} and \ref{eq:constraints} become obsolete and can thus be removed from $f^e$. This reduces the number of clauses in the encoding by:

\[\sum_{x \in X} \underbrace{\vphantom{\binom{n}{2}}1}_{\substack{ALO\\clause}} + \underbrace{
\binom{n}{2}}_{\substack{AMO\\clauses}},\]

\noindent where we sum over every $x \in X$, with $n$ the domain size of $x$, i.e., $|\mathcal{A}(x)|$. Both the at-least-once (AMO) and at-most-once (AMO) clauses contribute to reducing the number of clauses in the encoding to the number of probabilities in the CPTs of the BN. This gives an advantage over related work using the direct encoding, as it puts less strain on the SAT-solver by requiring it to only process $\mathcal{M}(f)$.

\subsection{Compiler}
The compiler uses a lazy SMT-solver, consisting of a theory- and SAT-solver. In the way we have build the compiler, it naturally allows for optimization by providing the ability to substitute the SAT-solver with any other state-of-the-art solver. We have optimized the SMT-solver by using the structure expressed by the encoding, and the fact that theory $\mathcal{T}$ and $f^e$ are defined over the same variables.

We have optimized the theory-solver such that it now supports \emph{constant time conditioning} of constraint clauses, which can take up one third of the encoding as shown by experimental results later. All one needs is the function $V: \mathcal{A}(X) \rightarrow X$ that maps literal $x_i$ back to $x$, where $x_i \in \mathcal{A}(x)$. For each $x$ we maintain a counter that is initialized to the domain size of $x$, i.e., $|\mathcal{A}(x)|$. Conditioning on negated literal $\overline{x_i}$ will decrease the counter corresponding to $x$ by 1. If the counter reaches 0, we derive contradiction (i.e., unsatisfiable as $x$ has no value). Conditioning on positive literal $x_i$ will cause any following conditioning on $x_j \in \mathcal{A}(x){\backslash}x_i$ as redundant (i.e., $x$ can only have one value). The SAT-solver will be bypassed completely as a result and the compiler will continue with the next variable in the ordering, saving additional time.

We have simplified the SAT-solver considerably by taking advantage of the structure of $\mathcal{M}(f)$. We use an one-to-many map $Q: \mathcal{A}(X) \rightarrow O$ from literal $l \in \mathcal{A}(X)$ to the clauses $O$ it occurs in, i.e., $Q(l) = \{c^1,\dotsc,c^n\}$. For each clause $c^i$, we maintain if it is satisfied with a counter, initialized to the number of literals it consists of. When conditioning on positive literal $l$ we decrease counters associated with $Q(l)$ by 1. The clauses of which the counters have reached 0 are marked as satisfied, and their corresponding weights are set aside to be introduced into the representation later. Conditioning on negated literal $\overline{l}$ will mark clauses $Q(\overline{l})$ as satisfied. This is possible because $\mathcal{M}(f)$ only contains negated literals, and we are able to assume that $Q(l) \cap Q(\overline{l}) = \emptyset$. When all clauses in $f^m$ are satisfied, we derive $f$ to be satisfiable given the evaluated instantiation. In combination with the SAT-solver being bypassed in the case of the previously mentioned redundant variables, this allows for conditioning in \emph{linear time}, in the number of clauses that $l$ occurs in.

\section{Experimental Results}\label{sec:results}

\begin{table}[t!]
    \centering
    \begin{small}
        \begin{tabular} {l| r r r r r r r}                                      
            Bayesian\\ Network & $X$ &$\mathcal{A}(X)$ & $C$& $P$ & $P^u$ & $P^d$ \\\hline      
            &&&&&&&\\[-2ex]
            example         & 2    & 5    & 6     & 8      & 3     & 5      \\  
            cancer          & 5    & 10   & 10    & 20     & 20    & 0      \\  
            earthquake      & 5    & 10   & 10    & 20     & 20    & 0      \\  
            asia            & 8    & 16   & 16    & 36     & 27    & 9      \\  
            survey          & 6    & 14   & 16    & 37     & 37    & 0      \\  
            student farm    & 12   & 25   & 26    & 70     & 44    & 26     \\  
            sachs           & 8    & 24   & 32    & 228    & 175   & 53     \\  
            poker           & 7    & 43   & 145   & 748    & 71    & 677    \\  
            child           & 20   & 60   & 93    & 344    & 161   & 183    \\  
            carpo           & 54   & 122  & 139   & 554    & 246   & 308    \\  
            powerplant      & 40   & 120  & 160   & 432    & 360   & 72     \\  
            alarm           & 37   & 105  & 143   & 752    & 182   & 570    \\  
            win95pts        & 76   & 152  & 152   & 1148   & 274   & 874    \\  
            insurance       & 27   & 89   & 142   & 1419   & 427   & 992    \\  
            andes           & 220  & 440  & 440   & 2308   & 652   & 1656   \\  
            hepar2          & 70   & 162  & 190   & 2139   & 1922  & 217    \\  
            hailfinder      & 56   & 223  & 470   & 3741   & 835   & 2906   \\  
            pigs            & 441  & 1323 & 1764  & 8427   & 1474  & 6953   \\  
            link            & 714  & 1793 & 2304  & 20462  & 1282  & 19180  \\  
            water           & 32   & 116  & 188   & 13484  & 3578  & 9906   \\  
            munin1          & 186  & 992  & 3522  & 19226  & 4323  & 14903  \\  
            pathfinder      & 135  & 520  & 3600  & 106432 & 2379  & 104053 \\  
            weeduk          & 15   & 90   & 347   & 22611  & 4600  & 18011  \\  
            fungiuk         & 15   & 165  & 1144  & 43007  & 8990  & 34017  \\  
            munin2          & 1003 & 5376 & 19460 & 83920  & 23228 & 60692  \\  
            munin3          & 1041 & 5601 & 20292 & 85615  & 24495 & 61120  \\  
            munin           & 1041 & 5651 & 20432 & 98423  & 24222 & 74201  \\  
            munin4          & 1038 & 5645 & 20426 & 97943  & 24621 & 73322  \\  
            mildew          & 35   & 616  & 17550 & 547158 & 14772 & 532386 \\  
            mainuk          & 48   & 421  & 3607  & 130180 & 18883 & 111297 \\  
            diabetes        & 413  & 4682 & 31738 & 461069 & 17888 & 443181 \\  
            barley          & 48   & 421  & 3607  & 130180 & 36924 & 93256  \\  
        \end{tabular}
    \end{small}
    \caption{Various statistics on BNs and their encoding, where the number of variables $X$, literals $\mathcal{A}(X)$, constraint clauses $C$, probabilities $P$, cumulative amount of unique probabilites per CPT $P^u$ and the number of deterministic probabilities $P^d$ are shown.}
    \label{tab:results}
\end{table}

We have developed a tool chain, that can encode a Bayesian network into CNF, compile it to various different representations, and perform inference using the arithmetic circuits they induce. Using over 30 publicly available Bayesian networks, we provide empirical results on encoding size, representation size and compilation time comparisons to other well known representations and compilers. We also compare the time it takes to perform exact inference compared to the classic Junction tree algorithm.

Statistics related to the encoding are shown in Table~\ref{tab:results}, which include Example~\ref{ex:full} as BN \verb+example+. The number of clauses produced by encoding $\mathcal{E}$ is equal to $|f^c|+|P|$, for constraint clauses $f^c$ and mapping $\mathcal{M}$, disregarding determinism. We can reduce the size of the encoding by up to a third, by moving constraint clauses to the theory solver, additionally allowing us to perform constant time conditioning on them. We can also see that the majority of the BNs will benefit greatly by the techniques in this paper by looking at the amount of equal and deterministic probabilities they contain.

We have developed a compiler that supports compilation of Bayesian networks to OBDDs and ZBDDs (using the CUDD\footnote{Available at http://vlsi.colorado.edu/{\raise.17ex\hbox{$\scriptstyle\sim$}}fabio/} 3.0.0 library), SDDs \cite{darwiche2011sdd} (using the SDD\footnote{Available at http://reasoning.cs.ucla.edu/sdd/} 1.1.1 library), and WPBDDs. Each decision diagram is created with the same ordering, within the same framework, i.e., doing the same amount of work in the same order. Quite literally, the only differences are the inserted appropriate function calls to different libraries, and the output representation. This will have comparative implications to whether a particular compilation will succeed given resource constraints as time and memory. At the same time, we did not tune the algorithms to ensure that our algorithm stood out favorably, ensuring fair comparison.

The framework divides the compilation process in two for efficiency. The logical representation of each CPT is first compiled separately, and then conjoined to represent the full distribution. The later is essentail for producing a logical circuit with a consistent model count in order to perform for inference. All results regarding the WPBDD compiler have been produced with a hybrid approach, where CPTs are compiled in a topdown fashion, and conjoined bottom-up. We found that a fully bottomup approach is only favorable when BNs have large CPTs like \verb+mildew+, where we got a 5x speedup compared to the topdown approach. In practice, large CPTs are usually avoid as they increase the complexity of inference.

Many strategies were explored in order to find a good variable ordering for each BN. Using simulated annealing in combination with an upper bound function yielded best results by far. The variable orderings were used to induce orderings based on literals, by saying that literal $x_1$ must come before $y_1$ if variable $x$ comes before $y$ in the variable ordering, where $x,y \in X$, $x_1 \in \mathcal{A}(x)$ and $y_1 \in \mathcal{A}(y)$. The weights are introduced into the ordering as literals precisely when the WPBDD would introduce them as edge weights.

Tables~\ref{tab:sizes} and \ref{tab:runtimes} show a comparative study between representation size and compilation time of supported representations, where WPBDD$^{nc}$ is a WPBDD where the collapse rule has not been applied, in order to show the impact that this rule has. SDDs and SDD$^r$s are compiled using a balanced and right-aligned vtree ordering, respectively. A left-to-right traversal of these vtrees produces the ordering also used for the other representations. Table~\ref{tab:sizes} indicates compilation failure due to a 24Gb RAM memory limitation or an one hour time limit by symbols - and *, respectively. The progress each failed compilation made before is indicated in Table~\ref{tab:runtimes}. All experiments were run using an Intel Xeon E5620 CPU.

Table~\ref{tab:sizes} shows a size comparison of each representation by the only commom size metricperators in the logical circuit that each decision diagram induces. We can see that WPBDDs have 60\% less logical operators than the corresponding OBDDs on average at both stages of compilation, reducing inference time and system requirements considerably. Also, a WPBDD is reduced by 15\% on average by applying the collapse rule when compiling CPTs, and 6\% reduction on average with fully compiled networks. This statistic is fully determined by the amount of local structure in the BN and the ordering used during compilation, and can greatly be improved upon utilizing techniques as dynamic compilation in the future.

Observe that there is a close relation between then size of OBDDs and SDD$^r$s, as mentioned in \cite{darwiche2011sdd}. We can see that the size of each SDD$^r$ is marginally smaller than its corresponding OBDD in Table~\ref{tab:sizes}. We assume that this is because SDDs have multi-valued logical-OR operators, which allow for more concise representations. SDDs consist of binary logical-AND, and $n$-ary logical-OR operators. We have included OR operators in size computations as $n-1$ binary logical-OR operators.

\begin{landscape}

\begin{table}[t!]
    \centering
        \begin{small}
            \begin{tabular} {l| rrrrrr | rrrrrr }
                Bayesian & \multicolumn{6}{c|}{Number of operators per CPT} & \multicolumn{6}{c}{Total number of operators}\\
                Network  &
                    \footnotesize{WPBDD} &\footnotesize{WPBDD$^{nc}$}& \footnotesize{OBDD} & \footnotesize{ZBDD}  & \footnotesize{SDD$^r$} & \footnotesize{SDD} &
                    \footnotesize{WPBDD} &\footnotesize{WPBDD$^{nc}$} & \footnotesize{OBDD} & \footnotesize{ZBDD}  & \footnotesize{SDD$^r$} & \footnotesize{SDD} \\ \hline
                &&&&&&&&&&&&\\[-2ex]
                example        &  \textbf{11}      &  19            &  48       &  54         & 33       & 49      &  \textbf{9}         &  15            &  39        &  30         &  27         &  49                 \\
                cancer         &  \textbf{80}      &  \textbf{80}   &  195      &  747        & 135      & 292     &  \textbf{66}        &  \textbf{66}   &  159       &  324        &  144        &  362                \\
                earthquake     &  \textbf{80}      &  \textbf{80}   &  195      &  747        & 135      & 292     &  \textbf{66}        &  \textbf{66}   &  159       &  324        &  144        &  362                \\
                survey         &  \textbf{147}     &  \textbf{147}  &  354      &  1671       & 270      & 510     &  \textbf{132}       &  \textbf{132}  &  312       &  825        &  291        &  819                \\
                asia           &  \textbf{131}     &  136           &  321      &  1473       & 231      & 425     &  \textbf{136}       &  \textbf{136}  &  321       &  564        &  306        &  765                \\
                student\_farm  &  \textbf{231}     &  252           &  591      &  3522       & 441      & 747     &  \textbf{459}       &  465           &  1098      &  2007       &  1080       &  2135               \\
                sachs          &  \textbf{747}     &  759           &  1788     &  20928      & 1611     & 3176    &  \textbf{630}       &  \textbf{630}  &  1602      &  13164      &  1581       &  4881               \\
                poker          &  \textbf{731}     &  1128          &  2373     &  6774       & 2289     & 4219    &  \textbf{912}       &  1095          &  2370      &  5394       &  2355       &  6082               \\
                child          &  \textbf{1011}    &  1099          &  2739     &  26088      & 2385     & 4663    &  \textbf{2934}      &  2988          &  7872      &  21861      &  7857       &  16030              \\
                carpo          &  \textbf{1257}    &  1386          &  3264     &  77049      & 2574     & 4514    &  \textbf{2499}      &  3015          &  7179      &  16233      &  7164       &  13405              \\
                powerplant     &  \textbf{1414}    &  1558          &  3795     &  86184      & 3141     & 6382    &  \textbf{4158}      &  4362          &  11043     &  36276      &  11025      &  26662              \\
                alarm          &  \textbf{1553}    &  1701          &  3795     &  41688      & 3312     & 6183    &  \textbf{3832}      &  4294          &  10008     &  19227      &  9993       &  35004              \\
                hepar2         &  \textbf{8371}    &  8467          &  18528    &  1593654    & 16101    & 29584   &  \textbf{56574}     &  56871         &  142806    &  2658189    &  142791     &  188453             \\
                weeduk         &  \textbf{29196}   &  30543         &  70863    &  13762830   & 69135    & 170330  &  \textbf{32114}     &  34832         &  109734    &  23840949   &  109455     &  -                  \\
                fungiuk        &  \textbf{68430}   &  81435         &  295458   &  88536714   & 287940   & 250500  &  \textbf{234715}    &  251739        &  733551    &  95350635   &  727209     &  -                  \\
                win95pts       &  \textbf{1948}    &  2020          &  4722     &  122244     & 3822     & 6571    &  \textbf{312191}    &  320183        &  743631    &  1415904    &  743619     &  426550             \\
                insurance      &  \textbf{2776}    &  3171          &  7455     &  106959     & 6810     & 11836   &  \textbf{468514}    &  474682        &  1263420   &  5465610    &  1263393    &  1731502            \\
                pathfinder     &  \textbf{23366}   &  40838         &  240915   &  6158832    & 237549   & 295272  &  \textbf{1899921}   &  2135180       &  5732988   &  29435283   &  5732976    &  2287777            \\
                hailfinder     &  \textbf{8909}    &  9520          &  25371    &  457860     & 23865    & 35707   &  11141550           &  11270157      &  31493220  &  137548032  &  31493157   &  \textbf{10508499}  \\
                water          &  \textbf{23498}   &  25458         &  56268    &  4776201    & 53979    & 81614   &  \textbf{18460995}  &  18561108      &  44005977  &  *          &  -          &  -                  \\
                mildew         &  \textbf{139386}  &  833215        &  1482432  &  56435169   & -        & 832829  &  \textbf{21698281}  &  22322743      &  71621388  &  *          &  -          &  -                  \\
                andes          &  \textbf{4592}    &  5546          &  11667    &  1125528    & 9192     & 13809   &  -                  &  -             &  -         &  -          &  -          &  -                  \\
                mainuk         &  \textbf{232466}  &  240710        &  577866   &  216698319  & 567135   & *       &  -                  &  -             &  *         &  *          &  -          &  -                  \\
                barley         &  \textbf{244447}  &  252935        &  649902   &  581976891  & 637101   & *       &  -                  &  -             &  *         &  -          &  -          &  -                  \\
                munin          &  \textbf{145258}  &  179744        &  435102   &  363002097  & 415134   & 723740  &  -                  &  -             &  -         &  -          &  -          &  -                  \\
                munin4         &  \textbf{145373}  &  180658        &  432672   &  345413400  & 413277   & 733401  &  -                  &  -             &  *         &  -          &  -          &  -                  \\
                munin3         &  \textbf{140318}  &  172843        &  420537   &  327439035  & 400932   & 702203  &  *                  &  *             &  -         &  -          &  -          &  -                  \\
                diabetes       &  \textbf{186630}  &  288075        &  707679   &  111194403  & 696837   & 984197  &  -                  &  -             &  -         &  -          &  -          &  -                  \\
                pigs           &  \textbf{19953}   &  22482         &  49872    &  6976569    & 44127    & 68288   &  -                  &  -             &  *         &  -          &  -          &  -                  \\
                munin1         &  \textbf{25415}   &  32188         &  74613    &  7756470    & 71538    & 131514  &  -                  &  -             &  -         &  -          &  -          &  *                  \\
                link           &  \textbf{22869}   &  27249         &  60279    &  12387426   & 54291    & 85356   &  -                  &  -             &  -         &  -          &  -          &  -                  \\
                munin2         &  \textbf{133493}  &  165853        &  398550   &  323654088  & 379137   & 670702  &  -                  &  -             &  -         &  -          &  -          &  -                  \\
            \end{tabular}
        \end{small}
        \caption{Number of arithmetic operators in intermediate and resulting decision diagrams\\ (symbols - and * indicate compilation failure due to memory or an one hour time limitation, respectively)}
    \label{tab:sizes}
\end{table}

\end{landscape}

\begin{landscape}

    \begin{table}[t!]
    \centering
    \newcommand{\pr}[1]{\scriptsize{(#1\%)}}
    \begin{small}
        \begin{tabular} {l| rrrrrr | rrrrrr}
            Bayesian & \multicolumn{6}{c|}{CPT compile time} & \multicolumn{6}{c}{Conjoin compile time} \\
            Network  &
                \footnotesize{WPBDD} & \footnotesize{WPBDD$^{nc}$} & \footnotesize{OBDD} & \footnotesize{ZBDD} & \footnotesize{SDD$^r$} & \footnotesize{SDD} &
                \footnotesize{WPBDD} & \footnotesize{WPBDD$^{nc}$} & \footnotesize{OBDD} & \footnotesize{ZBDD} & \footnotesize{SDD$^r$}  & \footnotesize{SDD} \\ \hline 
            &&&&&&&&&&&&\\[-2ex]
            example        &  \textbf{0.000}  &  \textbf{0.000}  &  \textbf{0.000}    &  \textbf{0.000}  &  \textbf{0.000}  &  \textbf{0.000}  &  \textbf{0.000}   &  \textbf{0.000}  &  \textbf{0.000}    &  \textbf{0.000}  &  \textbf{0.000}  &  \textbf{0.000}  \\
            cancer         &  \textbf{0.000}  &  \textbf{0.000}  &  \textbf{0.000}    &  \textbf{0.000}  &  \textbf{0.000}  &  \textbf{0.000}  &  \textbf{0.000}   &  \textbf{0.000}  &  \textbf{0.000}    &  \textbf{0.000}  &  \textbf{0.000}  &  \textbf{0.000}  \\
            earthquake     &  \textbf{0.000}  &  \textbf{0.000}  &  \textbf{0.000}    &  \textbf{0.000}  &  \textbf{0.000}  &  \textbf{0.000}  &  \textbf{0.000}   &  \textbf{0.000}  &  \textbf{0.000}    &  \textbf{0.000}  &  \textbf{0.000}  &  \textbf{0.000}  \\
            survey         &  \textbf{0.000}  &  \textbf{0.000}  &  \textbf{0.000}    &  \textbf{0.000}  &  \textbf{0.000}  &  \textbf{0.000}  &  \textbf{0.000}   &  \textbf{0.000}  &  \textbf{0.000}    &  \textbf{0.000}  &  \textbf{0.000}  &  \textbf{0.000}  \\
            asia           &  \textbf{0.000}  &  \textbf{0.000}  &  \textbf{0.000}    &  \textbf{0.000}  &  \textbf{0.000}  &  \textbf{0.000}  &  \textbf{0.000}   &  \textbf{0.000}  &  \textbf{0.000}    &  \textbf{0.000}  &  \textbf{0.000}  &  \textbf{0.000}  \\
            student\_farm  &  \textbf{0.000}  &  \textbf{0.000}  &  \textbf{0.000}    &  \textbf{0.000}  &  \textbf{0.000}  &  \textbf{0.000}  &  \textbf{0.000}   &  0.001           &  \textbf{0.000}    &  \textbf{0.000}  &  \textbf{0.000}  &  0.001           \\
            sachs          &  \textbf{0.000}  &  \textbf{0.000}  &  \textbf{0.000}    &  \textbf{0.000}  &  0.002           &  0.007           &  0.001            &  0.001           &  \textbf{0.000}    &  0.001           &  \textbf{0.000}  &  0.012           \\
            poker          &  \textbf{0.001}  &  \textbf{0.001}  &  0.002             &  0.002           &  0.015           &  0.030           &  0.002            &  0.004           &  \textbf{0.000}    &  \textbf{0.000}  &  0.001           &  0.016           \\
            child          &  \textbf{0.000}  &  \textbf{0.000}  &  0.001             &  \textbf{0.000}  &  0.004           &  0.009           &  0.004            &  0.004           &  \textbf{0.002}    &  0.007           &  0.010           &  0.041           \\
            carpo          &  \textbf{0.001}  &  \textbf{0.001}  &  0.003             &  0.002           &  0.006           &  0.008           &  0.005            &  0.006           &  \textbf{0.004}    &  0.012           &  0.018           &  0.033           \\
            powerplant     &  \textbf{0.001}  &  \textbf{0.001}  &  0.002             &  0.002           &  0.005           &  0.006           &  0.007            &  0.008           &  \textbf{0.003}    &  0.025           &  0.018           &  0.045           \\
            alarm          &  \textbf{0.001}  &  \textbf{0.001}  &  0.002             &  0.002           &  0.008           &  0.018           &  0.006            &  0.007           &  \textbf{0.003}    &  0.008           &  0.015           &  0.059           \\
            hepar2         &  0.006           &  \textbf{0.005}  &  0.016             &  0.063           &  0.055           &  0.097           &  \textbf{0.072}   &  \textbf{0.072}  &  0.201             &  7.045           &  1.192           &  3.047           \\
            weeduk         &  0.815           &  0.817           &  \textbf{0.249}    &  1.178           &  1.615           &  2512.160        &  0.084            &  0.112           &  \textbf{0.013}    &  3.988           &  0.077           &  \pr{92.86}      \\
            fungiuk        &  1.799           &  1.803           &  \textbf{1.027}    &  11.011          &  7.276           &  179.470         &  126.701          &  356.738         &  \textbf{0.264}    &  1285.111        &  1.455           &  \pr{14.29}      \\
            win95pts       &  \textbf{0.003}  &  \textbf{0.003}  &  \textbf{0.003}    &  0.005           &  0.014           &  0.016           &  \textbf{0.129}   &  0.132           &  1.257             &  4.203           &  6.293           &  0.882           \\
            insurance      &  \textbf{0.002}  &  \textbf{0.002}  &  0.004             &  0.005           &  0.020           &  0.022           &  \textbf{0.201}   &  0.211           &  0.366             &  3.070           &  1.946           &  2.485           \\
            pathfinder     &  0.685           &  \textbf{0.678}  &  7.840             &  8.066           &  30.405          &  2.226           &  \textbf{1.333}   &  1.501           &  13.974            &  101.344         &  67.038          &  22.888          \\
            hailfinder     &  \textbf{0.011}  &  \textbf{0.011}  &  0.018             &  0.035           &  0.105           &  0.177           &  5.556            &  \textbf{5.546}  &  17.409            &  567.997         &  71.913          &  9.464           \\
            water          &  0.107           &  0.107           &  \textbf{0.090}    &  0.253           &  0.764           &  4.922           &  \textbf{44.400}  &  45.883          &  55.055            &  \pr{32.26}      &  \pr{96.77}      &  \pr{32.26}      \\
            mildew         &  275.356         &  275.445         &  \textbf{129.703}  &  134.639         &  \pr{11.43}      &  1146.805        &  304.752          &  2062.466        &  \textbf{160.524}  &  \pr{5.88}       &  \pr{0.00}       &  \pr{2.94}       \\
            andes          &  \textbf{0.006}  &  \textbf{0.006}  &  0.007             &  0.032           &  0.027           &  0.041           &  \pr{94.52}       &  \pr{94.52}      &  \pr{21.92}        &  \pr{21.00}      &  \pr{20.55}      &  \pr{21.92}      \\
            mainuk         &  7.503           &  7.528           &  \textbf{2.272}    &  27.709          &  20.179          &  \pr{4.17}       &  \pr{95.744}      &  \pr{95.74}      &  \pr{44.68}        &  \pr{8.51}       &  \pr{42.55}      &  \pr{0.00}       \\
            barley         &  7.682           &  7.683           &  \textbf{2.247}    &  77.280          &  27.284          &  \pr{4.17}       &  \pr{95.74}       &  \pr{93.62}      &  \pr{44.68}        &  \pr{2.13}       &  \pr{42.55}      &  \pr{0.00}       \\
            munin          &  0.193           &  \textbf{0.190}  &  0.458             &  30.999          &  21.165          &  2.568           &  \pr{75.67}       &  \pr{75.67}      &  \pr{1.83}         &  \pr{1.35}       &  \pr{1.73}       &  \pr{2.40}       \\
            munin4         &  \textbf{0.186}  &  0.189           &  0.438             &  28.238          &  21.256          &  2.560           &  \pr{78.98}       &  \pr{78.98}      &  \pr{1.54}         &  \pr{1.06}       &  \pr{1.35}       &  \pr{1.45}       \\
            munin3         &  \textbf{0.165}  &  0.173           &  0.395             &  26.503          &  23.430          &  2.225           &  \pr{74.23}       &  \pr{74.23}      &  \pr{2.50}         &  \pr{1.25}       &  \pr{1.63}       &  \pr{2.40}       \\
            diabetes       &  3.142           &  \textbf{3.137}  &  6.479             &  13.272          &  106.197         &  11.352          &  \pr{54.85}       &  \pr{54.85}      &  \pr{1.21}         &  \pr{0.73}       &  \pr{0.97}       &  \pr{1.21}       \\
            pigs           &  \textbf{0.016}  &  \textbf{0.016}  &  0.036             &  0.228           &  0.248           &  0.114           &  \pr{94.09}       &  \pr{94.09}      &  \pr{7.05}         &  \pr{6.59}       &  \pr{6.59}       &  \pr{8.18}       \\
            munin1         &  0.037           &  \textbf{0.036}  &  0.101             &  0.322           &  0.998           &  0.416           &  \pr{89.73}       &  \pr{89.73}      &  \pr{14.05}        &  \pr{11.89}      &  \pr{11.89}      &  \pr{12.97}      \\
            link           &  \textbf{0.040}  &  \textbf{0.040}  &  0.070             &  0.621           &  0.512           &  0.403           &  \pr{89.200}      &  \pr{89.20}      &  \pr{3.51}         &  \pr{2.95}       &  \pr{3.09}       &  \pr{3.23}       \\
            munin2         &  \textbf{0.162}  &  \textbf{0.162}  &  0.400             &  28.014          &  19.650          &  2.255           &  \pr{78.94}       &  \pr{78.94}      &  \pr{1.30}         &  \pr{0.60}       &  \pr{1.20}       &  \pr{1.098}      \\
        \end{tabular}
    \end{small}
    \caption{Compilation time in seconds.\\ (In case of compilation failure, the percentage of successfully conjoined/processed\\ variables is shown, of which the reason is documented in Table~\ref{tab:sizes})}
    \label{tab:runtimes}
\end{table}

\end{landscape}

\begin{table}[t!]
    \centering
    \begin{small}
        \begin{tabular} {l | r | r  r| r r r | r r   }
            Bayesian & Queries & \multicolumn{2}{c|}{WPBDD} & \multicolumn{3}{c|}{OBDD} & \multicolumn{2}{c}{Dlib} \\
            Network  &         & $T$ & $T_{total}$  & $T$ & $T_{total}$ & $S$ & $T$ & $S$ \\ \hline
            &&&&&&\\[-2ex]
example        &  17       &  \textbf{0.000}   & 0.000   &  0.000    & 0.000    &  1.20  &  0.001     &  34.70    \\
cancer         &  714      &  \textbf{0.001}   & 0.001   &  0.002    & 0.002    &  1.57  &  0.138     &  96.33    \\
earthquake     &  714      &  \textbf{0.001}   & 0.001   &  0.002    & 0.002    &  1.50  &  0.139     &  96.53    \\
survey         &  4448     &  \textbf{0.014}   & 0.014   &  0.023    & 0.023    &  1.70  &  2.244     &  165.40   \\
asia           &  18360    &  \textbf{0.053}   & 0.053   &  0.108    & 0.108    &  2.03  &  9.798     &  185.70   \\
sachs          &  258324   &  \textbf{2.142}   & 2.143   &  5.323    & 5.323    &  2.50  &  1471.963  &  687.23   \\
student\_farm  &  1109885  &  \textbf{7.346}   & 7.347   &  17.942   & 17.943   &  2.47  &  2627.214  &  357.73   \\
poker          &  145999     &  \textbf{2.204}   & 2.209   &  5.229    & 5.231    &  2.37  &  3574.706  &  1621.99  \\
child          &  121647     &  \textbf{4.058}   & 4.062   &  14.603   & 14.611   &  3.60  &  3545.921  &  873.81   \\
carpo          &  213257     &  \textbf{7.697}   & 7.704   &  21.340   & 21.355   &  2.77  &  3515.298  &  456.73   \\
powerplant     &  457407     &  \textbf{21.575}  & 21.584  &  64.450   & 64.477   &  2.99  &  3393.056  &  157.27   \\
alarm          &  94034      &  \textbf{5.062}   & 5.07    &  13.324   & 13.334   &  2.63  &  3547.504  &  700.77   \\
hepar2         &  26530      &  \textbf{16.419}  & 16.496  &  50.950   & 58.011   &  3.10  &  3463.509  &  210.95   \\
weeduk         &  1051       &  \textbf{0.420}   & 1.349   &  1.608    & 5.845    &  3.83  &  3586.316  &  8538.88  \\
fungiuk        &  718        &  \textbf{1.911}   & 360.452 &  6.556    & 1292.7   &  3.43  &  3567.136  &  1866.32  \\
win95pts       &  13595      &  \textbf{50.586}  & 50.721  &  135.639  & 139.845  &  2.68  &  3220.060  &  63.65    \\
insurance      &  280        &  \textbf{1.391}   & 1.604   &  4.477    & 7.551    &  3.22  &  3571.825  &  2567.81  \\
pathfinder     &  765        &  \textbf{23.501}  & 25.68   &  78.277   & 187.461  &  3.33  &  3391.518  &  144.32   \\
hailfinder     &  875        &  \textbf{133.381} & 138.938 &  432.411  & 1000.43  &  3.24  &  2448.653  &  18.36    \\
water          &  2221       &  \textbf{474.047} & 520.037 &  1319.978 & 1352.33  &  2.78  &  *         &  *        \\
mildew         &  1312       &  \textbf{290.731} & 2628.64 &  1401.271 & 1536.85  &  3.59  &  *         &  *        \\

            \hline         &        &        &  &          &  &\\[-2ex]
            Avg            Speedup  &        &  &        & &   &  2.69              &      &  991.82  \\
\end{tabular}
    \end{small}
    \caption{Inference time in seconds.}
    \label{tab:inference}
\end{table}

In order to evaluate the WMC approach to exact inference with other methods, we have chosen to compare to the Junction tree algorithm using the publicly available Dlib\footnote{Available at http://dlib.net/} C++ library (version 18.18), and the HUGIN\footnote{http://www.hugin.com} library (through the C++ API version 8.4). We exhaustively go through all possible probabilistic queries. Table~\ref{tab:inference} and \ref{tab:inferencehugin} show how much time $T$ spent by each method on an identical set of queries. We have excluded time spent on reading or processing the Bayesian network, as well as creating the join tree, purely focusing on inference time. We went through all possible queries up to \verb+poker+, and limited others by a reasonable amount of time. Time $T_{total}$ indicates the total time spent on compilation and inference, and shows that it occasionally depends on how many queries you intend to answer which language must be chosen to get results faster. In theory, the speedup $S$ of WPBDDs vs other logical representations coincides with the sizes difference of the arithmetic circuits they induce (see Table~\ref{tab:sizes}), as inference has linear complexity in the size of induced circuits. This is confirmed with an average speedup of over 2.6x compared to OBDDs. We achieved an average speedup of over 5x compared to HUGIN (Note that we were not able to process all BNs as we used the LITE (free) version, which comes with limitations). We also achieved a staggering speedup compared to the Junction tree algorithm by Dlib, confirmed by an exceptional amount of cache misses reported by cachegrind (Valgrind tool), and other profile information by GNU Gprof and GNU Perf on resource usage. Collectively, compile and inference results show that WPBDDs make a valuable addition in the field of exact probabilistic inference.

\begin{table}[!t]
    \centering
    \begin{small}
        \begin{tabular} {l | r | r | r r  }
            Bayesian & Queries & \multicolumn{1}{c|}{WPBDD} & \multicolumn{2}{c}{HUGIN} \\
            Network  &         & $T$ &  $T$ &  $S$ \\ \hline
            &&&&\\[-2ex]
            example        &  17      &  \textbf{0.000}  &  0.000   &  3.865  \\
            cancer         &  714     &  \textbf{0.002}  &  0.011   &  6.209  \\
            earthquake     &  714     &  \textbf{0.002}  &  0.012   &  6.390  \\
            survey         &  4448    &  \textbf{0.016}  &  0.085   &  5.279  \\
            asia           &  14742   &  \textbf{0.051}  &  0.407   &  8.057  \\
            sachs          &  148245  &  \textbf{1.347}  &  4.689   &  3.482  \\
            student\_farm  &  608263  &  \textbf{4.721}  &  29.264  &  6.199  \\
            poker          &  185881  &  \textbf{3.260}  &  4.929   &  1.512  \\
            \hline         &  ~       &  ~               &  ~       &  ~  \\[-2ex]
            Avg~Speedup     &  ~       &  ~               &  ~       &  5.124  \\[-2ex]\\
\end{tabular}
    \end{small}
    \caption{Inference time in seconds.}
    \label{tab:inferencehugin}
\end{table}

\section{Conclusion}\label{sec:conclusion}
To reduce the cost of Bayesian inference through Weighted Model Counting (WMC), we proposed a new canonical language called \emph{Weighted Positive Binary Decision Diagrams} that represent probability distributions more concisely. We have provided theoretical results in addition to practical results on compilation size with regard to 30+ Bayesian networks, where we have seen WPBDD induced logical circuits reduced by 60\% on average in comparison to OBDD induced circuits. The introduced reduction rule is responsibly for a 15\% reduction on average among compiled CPTs. These results can be improved upon even further in the future by finding a better variable ordering, which is made easier by WPBDDs, as they do not consider probabilities as auxiliary literals, reducing the search space considerably. We have evaluated the cost of inference compared to OBDD induced circuits, yielding a 2.5x speedup on average, and to the Junction tree algorithm, approaching a speedup of 1000x. The language thus gives computational benefits during model counting as well as compilation.

\section*{References}
\bibliographystyle{../style/elsevier/elsarticle-num}
\bibliography{paper.bib}

\newpage
\begin{appendices}

\theoremshannon*
\begin{proof}
 See \cite{brown1990}.
 \end{proof}
 \vspace{1em}

\theorempositiveshannon*

\begin{proof}
    We show that the theorem holds by reducing the Shannon expansion to the positive Shannon expansion through equivalence. Function $f$ is defined over variables $X = \{x\}$, and encoded as Boolean function $f^e =  f^c \land f^m$ given $\mathcal{E}$, where:
    \vspace{-1em}
    \begin{center}
        \[\mathcal{A}(x) = \{x_1, \ldots, x_n\},\ \ \ \ \ \ \hfill f^c = \left(\bigvee\limits_{x_j \in \mathcal{A}(x)} \overline{x_j}\right) \land \left(\bigwedge\limits_{x_k \in \mathcal{A}(x)}\ \bigwedge\limits_{x_l \in \mathcal{A}(x){\backslash}x_k}(\overline{x_k} \lor \overline{x_l})\right),\ \ \ \ \ \ \hfill f^m = 1.\]
    \end{center}

  It follows from Equation~\ref{eq:variables} and \ref{eq:constraints} that $f^c$ has a cardinality of at least two when encoding a non-trivial function, and that $f^m$ essentially depends on a (non-strict) subset of variables that $f^c$ essentially depends on. We therefore chose $f^m$ to be simplistic to make the following reductions more intuitive. Note that the \emph{at-most-once} clauses generated for $x_i$ are subsumed by $f^c$ at the final step.
    \vspace{1em}
\begin{longtable*}{L C L}
    f^e & = & f^c \land f^m\\
        & = & \underbrace{\left(\bigvee_{x_j \in \mathcal{A}(x)} x_j\right) \land \left(\bigwedge_{x_k \in \mathcal{A}(x)}\ \bigwedge_{x_l \in \mathcal{A}(x){\backslash}x_k}(\overline{x_k} \lor \overline{x_l})\right)}_{f^c}\ \land\ \underbrace{\vphantom{\left(\bigwedge_{x_j \in \mathcal{A}(x){\backslash}x_i}\right)}1}_{f^m}\\\newpage
        & = & x_i \land
        \underbrace{\left(\bigwedge_{x_j \in \mathcal{A}(x){\backslash}x_i} \overline{x_j}\right)}_{f^e_{|x_i}} \lor\ \overline{x_i} \land
        \underbrace{\left(\bigvee_{x_j \in \mathcal{A}(x){\backslash}x_i} \overline{x_j}\right) \land \left(\bigwedge_{x_k \in \mathcal{A}(x){\backslash}x_i}\ \bigwedge_{x_l \in \mathcal{A}(x){\backslash}x_{\{i k\}}}\!\!\!\!\!\!\!\!(\overline{x_k} \lor \overline{x_l})\right)}_{f^e_{|\overline{x_i}}}\\
        & & \text{\scriptsize{\textit{Shannon expansion of} $f^e$ \textit{on} $x_i$.}}\\
        & = & x_i \land \left(\bigwedge_{x_j \in \mathcal{A}(x){\backslash}x_i} \overline{x_j}\right) \land
        \underbrace{\vphantom{\left(\bigvee_{x_j \in \mathcal{A}(x){\backslash}x_i} \overline{x_j}\right)}1}_{f^e_{{\parallel}x_i}} \lor\ \overline{x_i} \land
        \underbrace{\left(\bigvee_{x_j \in \mathcal{A}(x){\backslash}x_i} \overline{x_j}\right) \land \left(\bigwedge_{x_k \in \mathcal{A}(x){\backslash}x_i}\ \bigwedge_{x_l \in \mathcal{A}(x){\backslash}x_{\{i k\}}}\!\!\!\!\!\!\!\!(\overline{x_k} \lor \overline{x_l})\right)}_{f^e_{|\overline{x_i}}}\\
        & & \text{\scriptsize{\textit{Positive Shannon expansion of} $f^e$ \textit{on} $x_i$. \textit{Equivalent by Definition~\ref{def:implicitconditioning}, and Equation~\ref{eq:equality}.}}}\\
        & = & x_i \land \left( \bigwedge_{x_j \in \mathcal{A}(x){\backslash}x_i} \overline{x_j}\right) \land\ f^e_{{\parallel}x_i}\ \ \lor\ \ \overline{x_i} \land f^e_{|\overline{x_i}}\\
        & = & \left(\bigwedge\limits_{x_j \in \mathcal{A}(x){\backslash}x_i} (\overline{x_i} \lor \overline{x_j})\right) \land\ \ \left(x_i \land f^e_{{\parallel}x_i}\ \ \lor\ \ f^e_{|\overline{x_i}}\right)\\
        & = & \vphantom{\left( \bigwedge_{x_j \in \mathcal{A}(x){\backslash}x_i} \overline{x_j}\right)}f^c\  \land\ \ \left(x_i \land f^e_{{\parallel}x_i}\ \ \lor\ \ f^e_{|\overline{x_i}}\right)
\end{longtable*}
\vspace{-3em}
\end{proof}
\vspace{1em}

\theoremreducedpositiveshannon*
\begin{proof}
        We first show what models are introduced by one reduced positive Shannon expansion, providing clear implications what models are introduced by $n$ expansions (i.e. a decomposition). We then show that $f^c$ can be factored out of the positive Shannon decomposed function, and used to remove the introduced models, as they are subsumed by $f^c$.

Let function $f$ and its encoding be defined as provided in the proof of Lemma~\ref{def:positiveshannon}. We will first show that applying one reduced positive Shannon expansion removes at-most-once (AMO) clauses related to the variable we expand.

\begin{normalsize}
    \begin{longtable*}{L C L}
     f^e & = & f^c \land f^m\\
        & = & \underbrace{\left(\bigvee_{x_j \in \mathcal{A}(x)} x_j\right) \land \left(\bigwedge_{x_k \in \mathcal{A}(x)}\ \bigwedge_{x_l \in \mathcal{A}(x){\backslash}x_k}(\overline{x_k} \lor \overline{x_l})\right)}_{f^c}\ \land\ \underbrace{\vphantom{\left(\bigwedge_{x_j \in \mathcal{A}(x){\backslash}x_i}\right)}1}_{f^m}\\
        & = & x_i \land
        \underbrace{\vphantom{\left(\bigvee_{x_j \in \mathcal{A}(x){\backslash}x_i} \overline{x_j}\right)}1}_{f^e_{{\parallel}x_i}} \lor\         \underbrace{\left(\bigvee_{x_j \in \mathcal{A}(x){\backslash}x_i} \overline{x_j}\right) \land \left(\bigwedge_{x_k \in \mathcal{A}(x){\backslash}x_i}\ \bigwedge_{x_l \in \mathcal{A}(x){\backslash}x_{\{i k\}}}\!\!\!\!\!\!\!\!(\overline{x_k} \lor \overline{x_l})\right)}_{f^e_{|\overline{x_i}}}\\
    & = & f^w\\
    \end{longtable*}
\end{normalsize}
\noindent

We have $M_j \models f^e$, with $j = \{1,\ldots,n\}$, where $M_j = \{\overline{x_1},\ldots,\overline{x_{j-1}},x_j,\overline{x_{j+1}},\ldots,\overline{x_n}\}$ and $M_k \models f^w$, with $k = \{1,\ldots,(n-1)+2^{(n-1)}\}$, where $M_k = \{\overline{x_1},\ldots,\overline{x_{k-1}},x_k,\overline{x_{k+1}},\ldots,\overline{x_n}\}$ for $1 \leq k \leq n$ and $k \neq i$, and $M_k = \{\dots,x_i,\ldots\}$ for the remainder. We conclude that $f^e \not\equiv f^w$, because $f^e$ has $n$ models, while $f^w$ has $(n-1)+2^{(n-1)}$. Expanding $x_i$ has caused any model containing $x_i$ to be true, i.e., model $\{\overline{x_1},\ldots,\overline{x_{i-1}},x_i,\overline{x_{i+1}},\ldots,\overline{x_n}\}$ has changed to the models $\{\ldots,x_i,\ldots\}$. It is precisely those additional models in $M_k$ that are not in $M_j$, which can be removed by AMO clauses created for $x_i$, i.e.:

\[
    f^e \equiv \left(\bigwedge\limits_{x_j \in \mathcal{A}(x){\backslash}x_i} (\overline{x_i} \lor \overline{x_j})\right)\ \land \ f^w.
\]

Where one expansion on $x_i$ removes AMO clauses related to $x_i$, a decomposition clearly removes AMO clauses related to all variables $\mathcal{A}(x)$. This holds regardless of ordering, as a positive Shannon decomposition of $f^e$ is guaranteed to produce an isomorphic representation due to the symmetric nature of the constraints. We can reduce the positive Shannon expansion to its reduced by form factoring out $f^c$:

\begin{normalsize}
\begin{eqnarray*}
    f^w & = & f^c \land f^m\\
        & = & \underbrace{\left(\bigvee_{x_j \in \mathcal{A}(x)} \overline{x_j}\right) \land \left(\bigwedge_{x_k \in \mathcal{A}(x)}\ \bigwedge_{x_l \in \mathcal{A}(x){\backslash}x_k}(\overline{x_k} \lor \overline{x_l})\right)}_{f^c}\ \land\ \underbrace{\vphantom{\left(\bigwedge_{x_j \in \mathcal{A}(x){\backslash}x_i}\right)}1}_{f^m}\\
        & = & f^c \land \left(x_1 \land
\underbrace{\vphantom{\left(\bigvee_{x_j \in \mathcal{A}(x){\backslash}x_1} \overline{x_j}\right)}1}_{f^e_{{\parallel}x_1}} \lor\         \underbrace{\left(\bigvee_{x_j \in \mathcal{A}(x){\backslash}x_1} \overline{x_j}\right) \land \left(\bigwedge_{x_k \in \mathcal{A}(x){\backslash}x_1}\ \bigwedge_{x_l \in \mathcal{A}(x){\backslash}x_{\{1 k\}}}\!\!\!\!\!\!\!\!(\overline{x_k} \lor \overline{x_l})\right)}_{f^e_{|\overline{x_1}}}\right)\\
& = & f^c \land (x_1 \lor f^c \land (x_2 \land
\underbrace{1}_{f^e_{{\parallel}x_2}}
    \lor\ \underbrace{\ldots}_{f^e_{|\overline{x_2}}}))\\
        & = & f^c \land (x_1 \lor f^c \land (x_2 \lor f^c \land (x_3 \land
\underbrace{1}_{f^e_{{\parallel}x_3}}
    \lor\ \underbrace{\ldots}_{f^e_{|\overline{x_3}}})))\\
        & = & \ldots\\
        & = & f^c \land (x_1 \lor\ (f^c \land (x_2 \lor (f^c \land \ldots x_{n-1} \lor (f^c \land x_n)))))\\
        & = & f^c \land (x_1 \lor \ldots \lor x_n)\\
        & = & f^c
\end{eqnarray*}
\end{normalsize}

This confirms when using reduced positive Shannon decompositions, equivalence is only achieved under domain closure constraints, as AMO clauses are subsumed by $f^c$.
\end{proof}

\vspace{1em}

\theoremequivalent*

\begin{proof}
Let $\mathcal{E}(f) = f^e$ be a Boolean function, with $f$ a boolean function defined over variables $X = \{x^1,\dotsc,x^n\}$. We show that an OBDD representing $f$ is equivalent to a PBDD representing $\mathcal{E}(f)$ by comparing their induced logical circuits, under the premis that the collapse rule does not apply distribution, but deletes literals. Note that this comparison requires $f$ to be a Boolean function, and thus each $x \in X$ is mapped to two atoms $\mathcal{A}(x) = \{x_1, x_2\}$. For every Shannon expansion on $f$, there is an equivalent positive Shannon expansion on $f^e$ (Figure~\ref{fig:expansions}).
 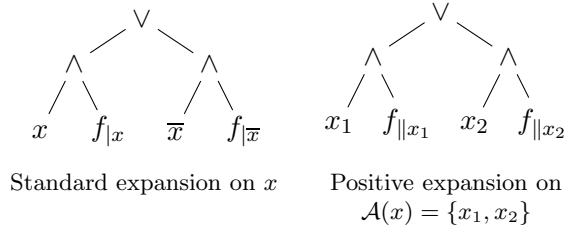
\begin{figure}[H]
    \centering
    \newcommand{\sepfig}{0.25}
    \begin{subfigure}[t]{\sepfig\textwidth}
        \centering
        \begin{tikzpicture}[
                scale=0.3,
                every path/.style={>=latex},
                every node/.style={draw,circle},
                inner sep=0pt,
                minimum size=0.5cm,
                line width=1pt,
                thin,
                font=\normalsize
                ]

                \node[draw=none] (or1) at (0,0) {$\vee$};
                \node[draw=none] (and1) at (-3,-2)  {$\wedge$};
                \node[draw=none] (and2) at (3,-2)  {$\wedge$};
                \node[draw=none] (a1) at (-4.5,-5)  {$x$};
                \node[draw=none] (a2) at (1.5,-5)  {$\overline{x}$};
                \node[draw=none] (alpha) at (-1.5,-5)     {$f_{|x}$};
                \node[draw=none] (beta) at (4.5,-5)     {$f_{|\overline{x}}$};

                \draw[-] (or1) edge (and1);
                \draw[-] (or1) edge (and2);
                \draw[-] (and1) edge (a1);
                \draw[-] (and2) edge (a2);
                \draw[-] (and1) edge (alpha);
                \draw[-] (and2) edge (beta);
        \end{tikzpicture}
        \caption*{Standard expansion on $x$}
    \end{subfigure}
    \begin{subfigure}[t]{\sepfig\textwidth}
        \centering
        \begin{tikzpicture}[
                scale=0.3,
                every path/.style={>=latex},
                every node/.style={draw,circle},
                inner sep=0pt,
                minimum size=0.5cm,
                line width=1pt,
                thin,
                font=\normalsize
                ]

                \node[draw=none] (or1) at (0,0) {$\vee$};
                \node[draw=none] (and1) at (-3,-2)  {$\wedge$};
                \node[draw=none] (and2) at (3,-2)  {$\wedge$};
                \node[draw=none] (a1) at (-4.5,-5)  {$x_1$};
                \node[draw=none] (a2) at (1.5,-5)  {$x_2$};
                \node[draw=none] (alpha) at (-1.5,-5)     {$f_{{\parallel}x_1}$};
                \node[draw=none] (beta) at (4.5,-5)     {$f_{{\parallel}x_2}$};

                \draw[-] (or1) edge (and1);
                \draw[-] (or1) edge (and2);
                \draw[-] (and1) edge (a1);
                \draw[-] (and2) edge (a2);
                \draw[-] (and1) edge (alpha);
                \draw[-] (and2) edge (beta);
        \end{tikzpicture}
        \caption*{Positive expansion on $\mathcal{A}(x) = \{x_1,x_2\}$}
    \end{subfigure}
    \caption{Expansions}
    \label{fig:expansions}
\end{figure}

We have equality, because there exists a trivial mapping between the two decomposition types, namely $x = x_1$ and $\overline{x} = x_2$. Observe the implication that the collapse rule can be applied to $f^e$ whenever the delete rule can be applied to $f$. Under our pretense we can infer that the OBDD and PBDD induce equivalent circuits, because they are isomorphic due to a trivial mapping for precisely those orderings where, for each $x \in X$, atoms $\mathcal{A}(x) = \{x_1,x_2\}$ are placed adjacent in the order, e.g., $\mathcal{A}(x^1) < \dotsc < \mathcal{A}(x^n)$ and for each $x \in X$ we have partial orders $x_1 < x_2$. The delete rule alters the circuit if cofactors are equal by applying the distributive law and identity, e.g., $x \land f\ \lor\ \overline{x} \land f = (x \lor \overline{x}) \land f = f$. The collapse rule forgoes that last step. We therefore conclude that the OBDD and PBDD induce equivalent circuits under Boolean identity.

\vspace{-1em}
\end{proof}

\vspace{1em}
\theoremequivalenttwo*
\begin{proof}

    In the case where $f$ is not Boolean, we will show that PBDD~$\varphi$ is always smaller than OBDD $\psi$, when both representing $\mathcal{E}(f) = f^e$. Consider some $x \in X$, for which we find atoms $\mathcal{A}(x) = \{x_1,x_2,x_3,\dotsc,x_n\}$ adjacent in the ordering, where $n$ is the domain size of $x$. OBDD $\psi$ will contain the subfunction shown in Figure~\ref{fig:inducedobdd}.

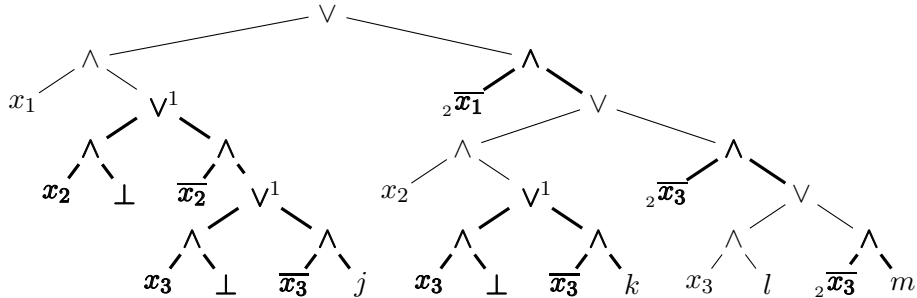
\begin{figure}[H]
    \centering
    \begin{subfigure}[t]{0.99\textwidth}
        \centering
        \begin{tikzpicture}[
                scale=0.3,
                every path/.style={>=latex},
                every node/.style={draw,circle},
                inner sep=0pt,
                minimum size=0.5cm,
                line width=1pt,
                thin,
                font=\normalsize
            ]

            \contourlength{0.15pt}
            \contournumber{10}

            \node[draw=none] (root) at (-15,4)   {$\lor$};
            \node[draw=none] (d1) at (-6,2)   {\contour{black}{$\land$}};
            \node[draw=none] (d2) at (-9,0){$_{_2}$\contour{black}{$\overline{x_1}$}};

            \begin{scope}[shift={(-12,0)}]
                \begin{scope}[shift={(-4.5,4)}]
                    \node[draw=none] (e1) at (-9,-2)   {$\land$};
                    \node[draw=none] (e2) at (-12,-4){$x_1$};

                    \node[draw=none] (c4) at (-6,-4)   {\hphantom{$^1$}\contour{black}{$\lor$}$^1$};
                    \node[draw=none] (c5) at (-9,-6)   {\contour{black}{$\land$}};
                    \node[draw=none] (c6) at (-10.5,-8){\contour{black}{$x_2$}};
                    \node[draw=none] (c7) at (-7.5,-8) {\contour{black}{$\bot$}};
                    \node[draw=none] (c8) at (-3,-6)   {\contour{black}{$\land$}};
                    \node[draw=none] (c9) at (-4.5,-8) {\contour{black}{$\overline{x_2}$}};
                \end{scope}

                \node[draw=none] (b4) at (-6,-4)   {\hphantom{$^1$}\contour{black}{$\lor$}$^1$};
                \node[draw=none] (b5) at (-9,-6)   {\contour{black}{$\land$}};
                \node[draw=none] (b6) at (-10.5,-8){\contour{black}{$x_3$}};
                \node[draw=none] (b7) at (-7.5,-8) {\contour{black}{$\bot$}};
                \node[draw=none] (b8) at (-3,-6)   {\contour{black}{$\land$}};
                \node[draw=none] (b9) at (-4.5,-8) {\contour{black}{$\overline{x_3}$}};
                \node[draw=none] (b10) at (-1.5,-8) {$j$};
            \end{scope}

            \node[draw=none] (a1) at (-3,0) {$\lor$};
            \node[draw=none] (a2) at (-9,-2) {$\land$};
            \node[draw=none] (a3) at (-12,-4)   {$x_2$};
            \node[draw=none] (a4) at (-6,-4)   {\hphantom{$^1$}\contour{black}{$\lor$}$^1$};
            \node[draw=none] (a5) at (-9,-6)   {\contour{black}{$\land$}};
            \node[draw=none] (a6) at (-10.5,-8){\contour{black}{$x_3$}};
            \node[draw=none] (a7) at (-7.5,-8) {\contour{black}{$\bot$}};
            \node[draw=none] (a8) at (-3,-6)   {\contour{black}{$\land$}};
            \node[draw=none] (a9) at (-4.5,-8) {\contour{black}{$\overline{x_3}$}};
            \node[draw=none] (a10) at (-1.5,-8) {$k$};
            \begin{scope}[shift={(-6,0)}]
                \node[draw=none] (a11) at (9,-2)   {\contour{black}{$\land$}};
                \node[draw=none] (a12) at (6,-4)    {$_{_2}$\contour{black}{$\overline{x_3}$}};
                \node[draw=none] (a13) at (12,-4)    {$\lor$};
                \node[draw=none] (a14) at (9,-6)    {$\land$};
                \node[draw=none] (a15) at (7.5,-8) {$x_3$};
                \node[draw=none] (a16) at (10.5,-8)  {$l$};
                \node[draw=none] (a17) at (15,-6)    {\contour{black}{$\land$}};
                \node[draw=none] (a18) at (13.5,-8)  {$_{_2}$\contour{black}{$\overline{x_3}$}};
                \node[draw=none] (a19) at (16.5,-8) {$m$};
            \end{scope}

            \draw[-] (root) edge (d1);
            \draw[-] (root) edge (e1);

            \draw[-] (e1) edge (e2);
            \draw[-] (e1) edge (c4);

            \draw[-,line width=1.2pt] (d1) edge (d2);
            \draw[-,line width=1.2pt] (d1) edge (a1);

            \draw[-,line width=1.2pt] (c4) edge (c5);
            \draw[-,line width=1.2pt] (c4) edge (c8);
            \draw[-,line width=1.2pt] (c5) edge (c6);
            \draw[-,line width=1.2pt] (c5) edge (c7);
            \draw[-,line width=1.2pt] (c8) edge (c9);
            \draw[-,line width=1.2pt] (c8) edge (b4);

            \draw[-,line width=1.2pt] (b4) edge (b5);
            \draw[-,line width=1.2pt] (b4) edge (b8);
            \draw[-,line width=1.2pt] (b5) edge (b6);
            \draw[-,line width=1.2pt] (b5) edge (b7);
            \draw[-,line width=1.2pt] (b8) edge (b9);
            \draw[-,line width=1.2pt] (b8) edge (b10);

            \draw[-] (a1) edge (a2);
            \draw[-] (a1) edge (a11);
            \draw[-] (a2) edge (a3);
            \draw[-] (a2) edge (a4);
            \draw[-,line width=1.2pt] (a4) edge (a5);
            \draw[-,line width=1.2pt] (a4) edge (a8);
            \draw[-,line width=1.2pt] (a5) edge (a6);
            \draw[-,line width=1.2pt] (a5) edge (a7);
            \draw[-,line width=1.2pt] (a8) edge (a9);
            \draw[-,line width=1.2pt] (a8) edge (a10);
            \draw[-,line width=1.2pt] (a11) edge (a12);
            \draw[-,line width=1.2pt] (a11) edge (a13);
            \draw[-] (a13) edge (a14);
            \draw[-] (a13) edge (a17);
            \draw[-] (a14) edge (a15);
            \draw[-] (a14) edge (a16);
            \draw[-,line width=1.2pt] (a17) edge (a18);
            \draw[-,line width=1.2pt] (a17) edge (a19);
        \end{tikzpicture}
    \end{subfigure}
    \caption{Constraints in OBDD $\psi$}
    \label{fig:inducedobdd}
\end{figure}

Here, $j,k,l$ and $m$ are subfunctions that essentially depend on $\{x_4,\ldots,x_n\}$. We can transition from the OBDD induced logical circuit above to the corresponding logical circuit induced by a PBDD, by removing the bold edges, literals and operators. Shannon expansions that have indicator $1$ at their root can be removed, which coincides with implicit conditioning. More generally, this allows us to remove $\sum^{n}_{i=1} (i~\!-~\!1)$ nodes from the OBDD. Additionally, the implicates for negative cofactors are removed that are marked by indicator~$2$.

Removing the implicate for the negative cofactor, and implicit conditioning, contribute to reducing the size of induced logical circuits by removing operators while maintaining equivalence. The extent to which is lower bounded by:

\[\sum_{x \in X} \underbrace{\vphantom{\sum^{n}_{i=1}}n}_{\substack{negative\\cofactor}} +\ \ \ \underbrace{\sum^{n}_{i=1} (i-1) * 3}_{\substack{implicit\\conditioning}},\]

\noindent where we sum over every $x \in X$, with $n$ the domain size of $x$, i.e., $|\mathcal{A}(x)|$. This increases when atoms $\mathcal{A}(x)$ are not adjacent in the ordering, and is multiplied by the number of distinct subfunctions in OBDD $\psi$ that depend on $\mathcal{A}(x)$. Furthermore, we are guaranteed to encounter each $x \in \mathcal{A}(X)$ in an OBDD along every path from root to the \emph{true} terminal, which serves as an upper bound regarding PBDDs. The collapse rule can effectively reduce the size of the representation if the structure of $f^e$ allows it, by reducing the number of nodes from $|\mathcal{A}(pa(x))|$ to $|pa(x)|$ for every subfunction, provided that literals $\mathcal{A}(y)$, with $y \in pa(x)$, for each parent are adjacent in the ordering.

\vspace{-1em}
\end{proof}

\end{appendices}

\end{document}